\documentclass[conference]{IEEEtran}
\IEEEoverridecommandlockouts
\usepackage{cite}
\usepackage{amsmath,amssymb,amsfonts}
\usepackage{algorithmic}
\usepackage{graphicx}
\usepackage{textcomp}
\usepackage{xcolor}
\def\BibTeX{{\rm B\kern-.05em{\sc i\kern-.025em b}\kern-.08em
    T\kern-.1667em\lower.7ex\hbox{E}\kern-.125emX}}

\usepackage{subcaption}
\usepackage{booktabs}
\usepackage{multirow}
\usepackage{cleveref}
\usepackage{xcolor}
\usepackage[table]{xcolor}
\usepackage{algorithm}
\usepackage{algorithmic}
\usepackage{amsthm}
\newtheorem{theorem}{Theorem}

\newtheorem{lemma}{Lemma}
\newcommand{\longname}{Proportional Fairness Attack}
\newcommand{\name}{PFA}

\newcommand{\hide}[1]{}

\begin{document}

\title{
Adversarial Bias: Data Poisoning Attacks on Fairness
\thanks{This work is supported by NSF (2134079).}
}

\author{\IEEEauthorblockN{Eunice Chan}
\IEEEauthorblockA{\textit{University of Illinois, Urbana-Champaign} \\
Champaign, USA \\
ecchan2@illinois.edu}
\and
\IEEEauthorblockN{Hanghang Tong}
\IEEEauthorblockA{\textit{University of Illinois, Urbana-Champaign} \\
Champaign, USA \\
htong@illinois.edu}
}

\maketitle

\begin{abstract}
With the growing adoption of AI and machine learning systems in real-world applications, ensuring their fairness has become increasingly critical. The majority of the work in algorithmic fairness focus on assessing and improving the fairness of machine learning systems. There is relatively little research on {\em fairness vulnerability}, i.e., how an AI system's fairness can be intentionally compromised. In this work, we first provide a theoretical analysis demonstrating that a simple adversarial poisoning strategy is sufficient to induce maximally unfair behavior in naive Bayes classifiers. Our key idea is to strategically inject a small fraction of carefully crafted adversarial data points into the training set, biasing the model's decision boundary to disproportionately affect a protected group while preserving generalizable performance. To illustrate the practical effectiveness of our method, we conduct experiments across several benchmark datasets and models. We find that our attack significantly outperforms existing methods in degrading fairness metrics across multiple models and datasets, often achieving substantially higher levels of unfairness with a comparable or only slightly worse impact on accuracy. Notably, our method proves effective on a wide range of models, in contrast to prior work, demonstrating a robust and potent approach to compromising the fairness of machine learning systems.
\end{abstract}



\section{Introduction}
Ensuring fairness in AI and machine learning systems is a critical concern alongside their growing real-world deployment. Current research efforts largely prioritize building \cite{mehrabi2021survey} and evaluating \cite{Kamiran2012, NIPS2016_9d268236} the fairness of such systems. In contrast, the vulnerability of existing systems to targeted adversarial attacks on their fairness is a crucial area that remains significantly under-explored.

Targeted adversarial attacks on fairness are a class of poisoning attacks that introduce adversarially crafted samples into a training dataset, leading to a system that exhibits worse fairness compared to one trained only on the clean samples. A common approach in prior work involves generating poisoned samples through loss optimization techniques \cite{DBLP:journals/corr/abs-2004-07401, liu2024towards, Mehrabi_Naveed_Morstatter_Galstyan_2021, nalmpantis2022re, kang2023deceptivefairnessattacksgraphs}. However, these methods rely on the assumption that the base system is differentiable, which limits their applicability. While some existing approaches do not require a differentiable base system, they fall short in fully exploiting the poisoned dataset available to the adversary \cite{Mehrabi_Naveed_Morstatter_Galstyan_2021, nalmpantis2022re}.
\begin{figure}
    \centering
    \includegraphics[trim={0 0.2cm 0 1.5cm},clip,width=\linewidth]{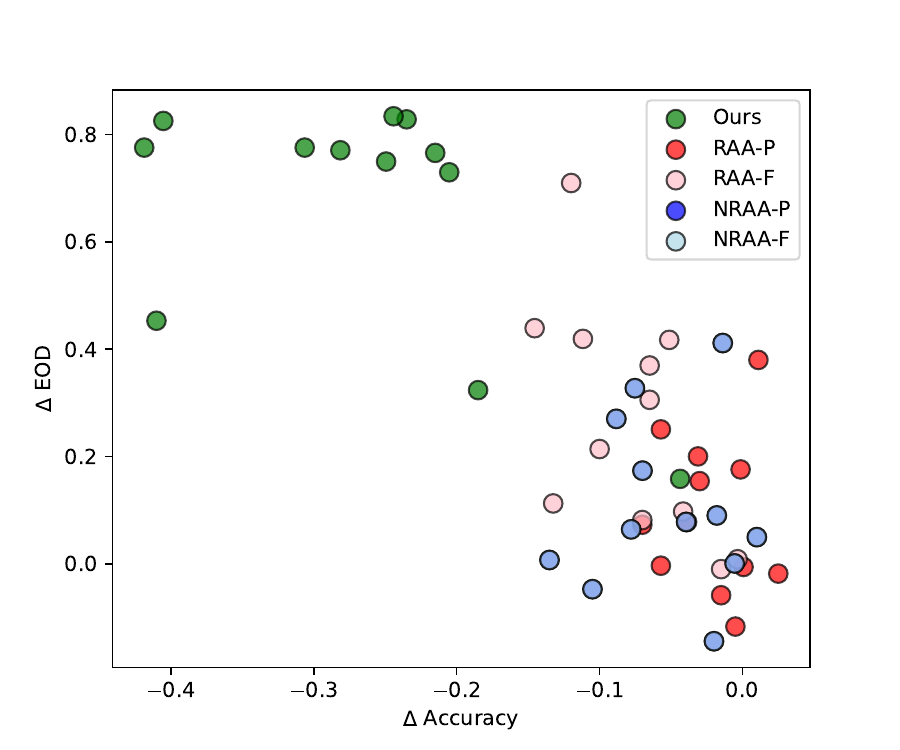}
    \caption{
    Our method consistently occupies a Pareto-dominant front,  achieving higher disparity (EOD) and better trade-offs compared to other methods.}
    \label{fig:summary}
\end{figure}
In this work, we introduce a novel fair data poisoning attack that does not assume a differentiable base system and more effectively utilizes the poisoned dataset to degrade fairness while maintaining the overall system performance. Our method strategically injects adversary-controlled samples into the dataset, reducing the fairness of a data-driven system trained on it. The key idea of our method is to leverage a surrogate model to identify and poison data points that are most influential in skewing the decision boundary against the targeted protected group, while carefully preserving overall predictive accuracy. Furthermore, we dynamically adjust the proportion of poisoned samples based on the surrogate model's intermediate outputs to maximize the fairness degradation effect. If fairness with respect to the attacked group is not actively monitored, this attack remains stealthy and may go unnoticed. 


To illustrate the practical effectiveness of our method, we benchmark our approach against existing methods across four base models. As visualized in \Cref{fig:summary}, our empirical results highlights the weakness of prior work in that it does not work well across a wide variety of base models. Furthermore, we demonstrate that our method is significantly more effective in attacking fairness than prior work: achieving the maximum fairness disparity for multiple models on several datasets. 
The main contributions of this paper are:
\begin{enumerate}
    \item \textit{Problem Formulation}: We formally define the problem of a targeted fairness poisoning attack against a machine learning system, aiming to maximize disparity with respect to a protected group while maintaining acceptable overall performance, without requiring differentiability of the base model.
    \item \textit{Theoretical Analysis}: We provide a formal theoretical analysis of the vulnerability of naive Bayes classifiers to the key ideas of our fairness attack strategy. We prove that a poisoning strategy with certain properties is guaranteed to induce maximally unfair behavior, establishing a theoretical foundation for the existence and efficacy of such attacks.
    \item \textit{Novel Attack Method}: We propose a novel and effective fair data poisoning attack, \longname{} (\name{}), which leverages a surrogate model and a dynamic proportional sampling strategy to strategically craft and inject poisoning samples, achieving significant fairness degradation with minimal impact on accuracy.
    \item \textit{Extensive Evaluation}: Through comprehensive experiments on multiple real-world datasets and a diverse set of base models, we demonstrate the superior performance and robustness of our approach compared to existing state-of-the-art methods, highlighting their limitations and providing valuable insights into the vulnerabilities of machine learning systems to fairness attacks.
\end{enumerate}

The rest of the paper is organized as follows: first, we review related work on data poisoning attacks and fairness in machine learning. Then, we detail our proposed fairness poisoning attack method. Following that, we present the experimental results and discussion of the results. Finally, we conclude the paper and outline directions for future work. 
\section{Related Work}

\subsection{Fair Machine Learning}
Much of the research in fair machine learning focus on defining measures of fair machine learning \cite{Kamiran2012, NIPS2016_9d268236} and developing algorithms to create fair systems \cite{mehrabi2021survey}. Such fairness algorithms can be broadly categorized into three stages: pre-processing, in-processing, and post-processing.

Pre-processing techniques pre-processes the data before training the classifer. This is frequently done with reweighing, subsampling, or feature engineering \cite{doi:10.1089/big.2016.0048, kamiran2012data, NEURIPS2022_ad991bbc}. In-processing techniques modify the training process itself, such as by adding fairness regularizers to the loss function \cite{agarwal2018reductions, pmlr-v28-zemel13, 10.1007/978-3-642-33486-3_3} or designing fairness-aware model architectures \cite{5694053, Calders2010, iosifidis2019adafair}. Finally, post-processing techniques can be done after the training process by adjusting the outputs directly \cite{doi:10.1089/big.2016.0048, 10.1145/3306618.3314287, 8682620}.

While extensive work has focused on developing fair machine learning methods, these interventions must be intentionally applied. In practice, many deployed models operate without any such modifications. This reality motivates our focus on base models when examining how data poisoning can systematically degrade fairness.  Furthermore, studying these unmodified base models reveals worst-case scenarios that fairness methods aim to mitigate. This motivates our work, which explores the impact of data poisoning on fairness and proposes a novel attack that effectively reduces fairness while minimizing the effect on overall accuracy.


\subsection{Fair Data Poisoning}
Fair data poisoning attacks aim to reduce the fairness of a machine learning model by injecting carefully crafted samples into the training dataset. A common approach involves generating these poisoned samples through loss optimization techniques \cite{DBLP:journals/corr/abs-2004-07401, liu2024towards, Mehrabi_Naveed_Morstatter_Galstyan_2021, nalmpantis2022re, kang2023deceptivefairnessattacksgraphs, 10.1007/978-3-031-00123-9_30}. However, these methods typically require the base model to be differentiable, restricting their application to a limited set of scenarios. Existing approaches that do not rely on differentiability often fail to fully leverage the adversary’s control over the poisoned dataset \cite{Mehrabi_Naveed_Morstatter_Galstyan_2021, nalmpantis2022re, 10.1007/978-3-031-00123-9_30}. To address these limitations, we propose a novel fair data poisoning attack that operates without differentiability assumptions and effectively utilizes the poisoned dataset to degrade fairness while minimizing the impact on overall accuracy.

\section{Problem Statement}

\begin{table}[h]
\centering
\resizebox{0.5\textwidth}{!}{%
\begin{tabular}{ll}
\toprule
\textbf{Symbol} & \textbf{Description} \\
\midrule
$\mathcal{D}$ & Original training dataset: $(\mathbf{X}, S, Y)$ \\
$\mathcal{D}_e$ & Evaluation dataset \\
$\mathcal{D}_p$ & Poisoned dataset generated by adversary \\
$\mathbf{X}$ & Non-sensitive input features \\
$V_j$ & Set of possible values for the $j$-th non-sensitive input feature \\
$S \in \{0,1\}$ & Binary sensitive attribute \\
$Y \in \{0,1\}$ & Binary outcome label \\
$S=1$ & Advantaged group: $\mathbb{P}[Y=1|S=1] \geq \mathbb{P}[Y=1|S=0]$ \\
$N$ & Number of poisoning iterations (candidate datasets generated) \\
$m_{\mathcal{D} \cup \mathcal{D}'}$ & Classifier trained on $\mathcal{D} \cup \mathcal{D}'$ \\
$\varepsilon \in [0,1]$ & Poisoning budget (fraction of $\mathcal{D}$ used for poisoning) \\
$\alpha, \beta \geq 0$ & Parameters controlling fairness-performance trade-off \\
\bottomrule
\end{tabular}
}
\caption{Main notation}
\label{tab:notation}
\end{table}

In this paper, we focus on attacking the group fairness properties of a classifier which predicts a binary outcome ($Y$) using a set of $K$ non-sensitive features ($\mathbf{X} = (X_1, \dots, X_K)$) and one binary sensitive attribute ($S$). In particular, we focus on the statistical parity difference (SPD) \cite{Kamiran2012} and equalized odds difference (EOD) \cite{NIPS2016_9d268236} measures of fairness. The training dataset is notated as $\mathcal{D}$: ($\mathbf{X}$, $S$, $Y$). Subsets of the dataset are indicated using subscripts to specify the selection criteria. The labels associated with each sample in a dataset $\mathcal{D}$ is denoted as $\mathcal{Y}(\mathcal{D})$. We define the advantaged group $S=1$ as the group that has a higher proportion of positive outcomes in the dataset $\mathcal{D}$ ($\mathbb{P}[Y=1|S=1] \geq \mathbb{P}[Y=1|S=0]$). \Cref{tab:notation} summarizes the primary notation used throughout the paper.

The adversary has access to 
the dataset $\mathcal{D}$ with which to generate a poisoned dataset $\mathcal{D}_p$ with $\varepsilon|\mathcal{D}|$ samples where $\varepsilon \in [0, 1]$ controls the poisoning budget. The attacked classifier $m$, which takes ($\mathbf{X}$, $S$), as input, is then trained on $\mathcal{D} \cup \mathcal{D}_p$. The data poisoning attack aims to satisfy two dual goals: (1) to preserve overall classification performance so that if the fairness of the sensitive attribute group is not monitored, the attack will go undetected, and (2) to attack fairness, that is, increase the disparity in outcomes between the advantaged and disadvantaged groups.

\begin{algorithm}[tb]
\caption{\longname}
\label{alg:ours}
\textbf{Input}: 
Train Dataset $\mathcal{D}$, 
Evaluation Dataset $\mathcal{D}^e$, \\
Surrogate Model $m$, 
Trade-off Parameters $\alpha, \beta$, \\ 
Poison Proportion $\varepsilon$, 
Number of Candidates $N$ \\
\textbf{Output}: Poisoned Dataset $\mathcal{D}_p$
\begin{algorithmic}[1] 
\STATE $\mathcal{D}^c \gets [\emptyset, \dots, \emptyset]$ where the $i$-th item $\mathcal{D}^c[i]$ of the list represents the candidate dataset generated in the $i$-th iteration.
\STATE $\mathcal{B} \gets [b_0, \dots, b_k]$ s.t. $b_i \geq 0$ and  $\sum_{b \in \mathcal{B}} b = \varepsilon|\mathcal{D}|$.
\FOR{$n \in [1, \dots, N]$}
    \FOR{$B \in \mathcal{B}$}
        \STATE $\hat{Y} \gets m_{\mathcal{D} \cup \mathcal{D}^c[n]}(\mathcal{D}^e)$ \label{step:surrogate}
        \STATE Select a group for sampling with \Cref{eq:cdm}.
        \IF {$CDM_{n, 0} < CDM_{n, 1}$} \label{step:selcond}
            \STATE $\mathcal{D}^s \gets \mathcal{D}_{S=1, \hat{Y}=0}$
        \ELSE
            \STATE $\mathcal{D}^s \gets \mathcal{D}_{S=0, \hat{Y}=1}$
        \ENDIF \label{step:sampleend}
        \IF {$|\mathcal{D}^s| = 0$} \label{step:unifstart}
            \FOR{$b \in [1, \dots, B]$}
                \STATE $\mathcal{D}^s \gets x_{s,j} \sim Uniform(V_j)$
            \ENDFOR
        \ENDIF \label{step:unifend}
        \STATE $\mathcal{D}^c_{new} \gets \textsc{Sample}(\mathcal{D}^s, B)$ \label{step:sample}
        \STATE Create poisoned sample according to \Cref{thm:nb_unfair}.
        \STATE $\mathcal{D}^c[n+1] \gets \mathcal{D}^c[n] \cup \{(\mathbf{x}, s, s) \mid (\mathbf{x}, s) \in \mathcal{D}^c_{new}\}$  \label{step:flip}
    \ENDFOR
\ENDFOR
\STATE Select candidate dataset by assigning a score to each dataset according to \Cref{eq:selscore}.
\STATE \textbf{return} $\min_{\mathcal{D}' \in \mathcal{D}^c} f(m_{\mathcal{D} \cup \mathcal{D}'}(\mathcal{D}^e), \mathcal{Y}(\mathcal{D}^e))$ \label{step:sel}
\end{algorithmic}
\end{algorithm}

\section{Proposed Method}

We propose a poisoning strategy that incrementally constructs a candidate poisoned training dataset. In each iteration, our method, \longname\ (\name{}) (detailed in \Cref{alg:ours}) uses a surrogate model (Line \ref{step:surrogate}) to dynamically decide whether to target the advantaged or disadvantaged group at each step (Line \ref{step:selcond}), and generates poisoned samples in this sensitive attribute group (Lines \ref{step:sample}-\ref{step:flip}).

\subsection{Characterizing Maximal Unfairness}


The maximally unfair scenario occurs when the classifier completely segregates predictions by sensitive attribute e.g. $\hat{Y}=s$. We can demonstrate the severity of this unfairness by analyzing its impact on two standard fairness metrics: statistical parity difference (SPD) \cite{Kamiran2012} and equalized odds difference (EOD) \cite{NIPS2016_9d268236}.

SPD measures the difference in favorable outcomes between the advantaged and disadvantaged groups. The larger the difference, the less fair the predictor. It is defined as
\[
\text{SPD} = \left| \mathbb{P}(\hat{Y} = 1 \mid S = 1) - \mathbb{P}(\hat{Y} = 1 \mid S = 0) \right|.
\]


EOD measures the maximum difference in the true and false positive rates between the advantaged and disadvantaged groups. Like SPD, the larger the difference, the less fair the predictor. This difference is given by
\begin{align*}
\text{EOD} = &\max_{y \in \{0,1\}}  \Big| \mathbb{P}(\hat{Y} = 1 \mid Y = y, S = 1) \\
& - \mathbb{P}(\hat{Y} = 1 \mid Y = y, S = 0) \Big|.
\end{align*}


\begin{lemma}[Maximally Unfair Behavior] 
If a classifier exhibits $\hat{Y}=s$ behavior (i.e., its predictions perfectly align with the sensitive attribute), then both Statistical Parity Difference (SPD) and Equalized Odds Difference (EOD) attain their maximum value of 1.
\label{lem:max_unfair}
\end{lemma}
\begin{proof}
 We demonstrate this for each fairness metric:
\begin{enumerate} 
    \item \textit{SPD.}
    If $\hat{Y}$ is always equal to $S$, $\mathbb{P}(\hat{Y}=1 \mid S=1) = 1$ and $\mathbb{P}(\hat{Y}=1 \mid S=0) = 0$. Plugging these values into the SPD definition: $\text{SPD} = |1 - 0| = 1$.
    
    \item \textit{EOD.}
    Assuming both groups contain positive examples, consider the component for $y=1$. Given $\hat{Y}=s$, it follows:
        \begin{itemize}
            \item $\mathbb{P}(\hat{Y}=1 \mid Y=1, S=1) = 1$ since $S=1$ and $\hat{Y}=s$.
            \item $\mathbb{P}(\hat{Y}=1 \mid Y=1, S=0) = 0$ since $S=0$ and $\hat{Y}=s$.
        \end{itemize}
    EOD for $y=1$ is $|1-0|=1$. A similar derivation holds for the $y=0$. Hence, EOD is 1.
\end{enumerate}
\end{proof}

As demonstrated in \Cref{lem:max_unfair}, in the situation when $\hat{Y}=s$, both fairness measures attain their maximum value, indicating the classifier's behavior is maximally unfair, systematically favoring one group at the expense of the other.


\subsection{Poisoned Sample Generation}

The overall intuition for poisoned sample generation is that we select samples for which the surrogate model does not follow the maximally unfair behavior and flip the labels associated with them to encourage the model to do so. We use intuition drawn from the setting of a naive Bayes classifier base model to identify these samples.

Consider a naive Bayes classifier $m$ which assumes conditional independence between every pair of features given the outcome $y$. The decision rule to get the predicted outcome $\hat{y}$ is to select the $y$ with the largest posterior probability:
$$\hat{y} = m(\mathbf{x}, s) = {\arg \max}_{y\in\{0,1\}} \mathbb{P}(Y=y \mid \mathbf{x}, s)$$

Assuming conditional independence, this posterior probability is proportionate to $score_{NB}(\mathbf{x}, s, y)$, defined as:
$$\mathbb{P}[Y=y] \mathbb{P}[S=s \mid Y=y]\prod_{i=1}^{K} \mathbb{P}[X_i= x_i \mid Y = y].$$

The decision rule can then be rewritten as:
$$ {\arg \max}_{y\in\{0,1\}} score_{NB}(\mathbf{x}, s, y)$$


If an adversary seeks to push the naive Bayes classifier toward the maximally unfair behavior of predicting $\hat{Y}=s$ (\Cref{lem:max_unfair}) for all $\mathbf{x}$, the adversary would want to manipulate the probabilities so that for all $\mathbf{x}$ and a given sensitive attribute group $s$, $score_{NB}(\mathbf{x}, s, s) > score_{NB}(\mathbf{x}, s, 1-s)$. The adversary can achieve this by injecting poisoned samples, which we formalize in \Cref{thm:nb_unfair}.

\begin{theorem}[Inducing Maximally Unfair Behavior in Naive Bayes via Adversarial Score Biasing]
Suppose, for both sensitive attribute group $s$, a set of $k_s$ poisoned samples are generated, each of the form $(\mathbf{x}_s, S=s, Y=s)$, where the feature vector $\mathbf{x}_s = (x_{s,1}, \dots, x_{s,K})$ is constructed so that, as the number of added samples approaches infinity, the feature distributions become uncorrelated with the label. That is, for each feature $X_i$ and for $y \in Y$:
$$\lim_{k \to \infty} \mathbb{P}[X_i=x_i \mid Y=y] = \mathbb{P}[X_i=x_i].$$

Let a new classifier $\mathcal{M}_{p}$ be trained on $\mathcal{D}$ augmented with these poisoned samples, such that the class priors are balanced: $|\mathcal{D}_{Y=1}|+k_1 = |\mathcal{D}_{Y=0}|+k_0$.

There exists a sufficiently large number of poisoned samples such that the classifier $\mathcal{M}_{p}$ is guaranteed to exhibit maximally unfair behavior (e.g. $\hat{Y}=s$, as detailed by \Cref{lem:max_unfair}).
\label{thm:nb_unfair}
\end{theorem}
\begin{proof}
If $\mathcal{M}_{p}$ exhibits maximally unfair behavior for an input $(\mathbf{x}, s)$,
$$score_{NB}(\mathbf{x}, s, 1-s) < score_{NB}(\mathbf{x}, s, s).$$ To prove this, we show the ratio of the two scores is greater than 1.
$$1 < \frac{score_{NB}(\mathbf{x}, s, s)}{score_{NB}(\mathbf{x}, s, 1-s)}$$

\begin{lemma}[Prior Balance]
The ratio of class priors of the poisoned dataset, $\frac{\mathbb{P}[Y=s]}{\mathbb{P}[Y=1-s]}$, is 1.
\label{lem:score_1}
\end{lemma}
\begin{proof}
Let $|\mathcal{D}_{Y=y}'|$ denote the count of samples for class $y$ in the poisoned dataset ($|\mathcal{D}_{Y=s}| + k_{s}$). By the condition of the theorem, the counts for each class are equal: $|\mathcal{D}_{Y=s}'| = |\mathcal{D}_{Y=1-s}'|$. The total number of samples in the poisoned dataset is $|\mathcal{D}'| = |\mathcal{D}_{Y=s}'| + |\mathcal{D}_{Y=1-s}'| = 2|\mathcal{D}_{Y=s}'|$. Thus, the new class prior for $Y=s$ is: $$\mathbb{P}[Y=s] = \frac{|\mathcal{D}_{Y=s}'|}{|\mathcal{D}'_{total}|} = \frac{|\mathcal{D}_{Y=s}'|}{2|\mathcal{D}_{Y=s}'|} = 0.5.$$ A similar derivation holds for the new class prior for $Y=1-s$. Since both priors are equal to 0.5, their ratio is 1.
\end{proof}

\begin{lemma}[Group Posterior Divergence]
The ratio of group posteriors, $\frac{\mathbb{P}[S=s \mid Y=s]}{\mathbb{P}[S=s \mid Y=1-s]}$, diverges to infinity.
\label{lem:score_2}
\end{lemma}
\begin{proof}
To show this, we analyze the behavior of each posterior probability as the number of poisoned samples increases.

\begin{itemize}
    \item \textbf{Behavior of $\mathbb{P}[S=s \mid Y=s]$:} This term is estimated as $\frac{\mathcal{D}_{S=s, Y=s}}{\mathcal{D}_{Y=s}}$. With the addition of $k_s$ poisoned samples of the form $(\mathbf{x}, S=s, Y=s)$, both the numerator and denominator increases by $k_s$. As $k_s \to \infty$, this term approaches 1.

    \item \textbf{Behavior of $\mathbb{P}[S=s \mid Y=1-s]$:} This term is estimated as $\frac{\mathcal{D}_{S=s, Y=1-s}}{\mathcal{D}_{Y=1-s}}$. The numerator is unchanged since we only samples with the label matching the sensitive attribute group. However, the denominator, $N(Y=1-s)$, increases by $k_{1-s}$. As $k_{1-s} \to \infty$, the numerator remains constant while the denominator grows, causing this term to approach 0.
\end{itemize}

\noindent Since $\mathbb{P}[S=s \mid Y=s]$ will converge to 1 and $\mathbb{P}[S=s \mid Y=1-s]$ will converge to 0, the ratio also goes to infinity.
\end{proof}

\begin{lemma}[Feature-Label Independence]
For feature $X_i$, the ratio of the feature likelihood, $\frac{\mathbb{P}[X_i=x_i \mid Y=s]}{\mathbb{P}[X_i=x_i \mid Y=1-s]}$, approaches 1. 
\label{lem:score_3}
\end{lemma}
\begin{proof}
Recall the poisoning strategy ensures that, as the number of added samples approaches infinity, the features become uncorrelated with the label. That is, for each $X_i$ and all $y \in Y$:
$$\lim_{k \to \infty} \mathbb{P}[X_i=x_i \mid Y=y] = \mathbb{P}[X_i=x_i].$$

\noindent Thus, the ratio of the feature likelihood also approaches 1:
\[
\begin{aligned}
\lim_{k \to \infty} \frac{\mathbb{P}[X_i=x_i \mid Y=s]}{\mathbb{P}[X_i=x_i \mid Y=1-s]} &= \frac{\lim_{k \to \infty} \mathbb{P}[X_i=x_i \mid Y=s]}{\lim_{k \to \infty} \mathbb{P}[X_i=x_i \mid Y=1-s]} \\
&= \frac{\mathbb{P}[X_i=x_i]}{\mathbb{P}[X_i=x_i]} = 1.
\end{aligned}
\]
\end{proof}

\noindent The ratio of the scores $\frac{score_{NB}(\mathbf{x}, s, s)}{score_{NB}(\mathbf{x}, s, 1-s)}$ can be decomposed into the ratios of the class priors, the group posteriors, and the conditional feature likelihoods:
\[
    \frac{\mathbb{P}[Y=s] \mathbb{P}[S=s \mid Y=s]\prod_{i=1}^{K} \mathbb{P}[X_i= x_i \mid Y = s]}{\mathbb{P}[Y=1-s] \mathbb{P}[S=s \mid Y=1-s]\prod_{i=1}^{K} \mathbb{P}[X_i= x_i \mid Y = 1-s]}.
\]
Applying \Cref{lem:score_1}, \Cref{lem:score_2}, and \Cref{lem:score_3}, we can characterize the behavior as $k_s \to \infty$:

\begin{align*}
&\lim_{k \to \infty} \frac{\mathbb{P}[Y=s]}{\mathbb{P}[Y=1-s]} \frac{\mathbb{P}[S=s \mid Y=s]}{\mathbb{P}[S=s \mid Y=1-s]} \\
&\quad \times \prod_{i=1}^{K}\frac{\mathbb{P}[X_i= x_i \mid Y = s]}{\mathbb{P}[X_i= x_i \mid Y = 1-s]} \\
&= 1 \cdot \infty \cdot \prod_{i=1}^{K} 1 \\
&= \infty.
\end{align*}

\noindent Since $\infty > 1$, this poisoning strategy will converge to the maximally unfair behavior on the naive Bayes model.
\end{proof}

In \Cref{thm:nb_unfair}, we present a poisoning attack that, when given an unlimited budget, is guaranteed to force a naive Bayes classifier to exhibit maximally unfair behavior. The attack achieves this by carefully constructing the feature vector, increasing the correlation between the sensitive attribute and the label, and controlling the number of each type of sample added. This strategic addition of poisoned samples causes the posterior probabilities for the groups to diverge, while simultaneously neutralizing the influence of the class prior and feature likelihood terms. 

A critical component of this attack is the ability to construct feature vectors that asymptotically cause the likelihood ratio to converge to 1, $\frac{\mathbb{P}[X_i=x_i \mid Y=s]}{\mathbb{P}[X_i=x_i \mid Y=1-s]} \to 1$ (\Cref{lem:score_3}). This objective can be achieved through uniform sampling each feature in the feature vector $\mathbf{x}_s$, which we formalize as follows:

\begin{theorem}[Feature-Label Independence Through Uniform Feature Assignment]
Let $\mathcal{D}$ represent the clean dataset. An adversarial feature vector $\mathbf{x}_s = (x_{s,1}, \dots, x_{s,K})$ is constructed such that for each feature $X_j$, the value $x_{s,j}$ is drawn independently from a discrete uniform distribution over the set $V_j$. That is, for $v \in V_j$:
$$\mathbb{P}[x_{s,j} = v] = \frac{1}{|V_j|}.$$
As the number of these poisoned samples $k$ in the poisoned dataset $\mathcal{D}_p$ approaches infinity, the empirical distribution of feature $X_j$ in the augmented dataset $\mathcal{D}' = \mathcal{D} \bigcup \mathcal{D}_p$ becomes independent of the label $Y$. That is, for feature value $x_j \in V_j$ and any label $y \in Y$:
$$\lim_{k \to \infty} \mathbb{P}_{\mathcal{D}'}[X_j=x_j \mid Y=y] = \mathbb{P}_{\mathcal{D}'}[X_j=x_j].$$
\label{thm:unif_xs}
\end{theorem}
\begin{proof}
We prove this by showing that the empirical conditional probabilities on $\mathcal{D}'$ converge to the same value.

\begin{lemma}[Conditional Probability Converges to Uniform]
As the number of poisoned samples for label $y$, $k_y$, approaches infinity, the empirical conditional probability of feature $X_j$ on the augmented dataset $\mathcal{D}'$ converges to a constant value for each possible feature value, reflecting a uniform distribution:
$$\lim_{k_y \to \infty} \mathbb{P}_{\mathcal{D}'}[X_j=x_j \mid Y=y] = \frac{1}{|V_j|}.$$
\label{lem:cond_prob}
\end{lemma}
\begin{proof}
The empirical conditional probability is 
$$\mathbb{P}_{\mathcal{D}'}[X_j=x_j \mid Y=y] = \frac{|\mathcal{D}'_{X_j=x_j, Y=y}|}{|\mathcal{D}'_{Y=y}|}.$$
As $k_y$ goes to infinity, approximately $\frac{1}{|V_j|}$ of the $k_j$ samples added will have $X_j=x_j$ and $Y=y$. Thus, in the limit:
$$\lim_{k_y \to \infty}\frac{|\mathcal{D}'_{X_j=x_j, Y=y}| + k_j\frac{1}{|V_j|}}{|\mathcal{D}'_{Y=y}| + k_j} = \frac{1}{|V_j|}.$$
\end{proof}

\noindent The empirical marginal probability is
$$\mathbb{P}_{\mathcal{D}'}[X_j=x_j] = \sum_{y\in Y} \mathbb{P}_{\mathcal{D}'}[Y=y] \mathbb{P}_{\mathcal{D}'}[X_j=x_j \mid Y=y].$$
Using \Cref{lem:cond_prob} and the law of total probability, we show this probability converges to $\frac{1}{|V_j|}$ as the number of poisoned samples $k$ approaches infinity:

\begin{align*}
&\lim_{k \to \infty} \sum_{y\in Y} \mathbb{P}_{\mathcal{D}'}[X_j=x_j \mid Y=y]\mathbb{P}_{\mathcal{D}'}[Y=y] \\
&= \sum_{y\in Y} \left(\lim_{k \to \infty} \mathbb{P}_{\mathcal{D}'}[X_j=x_j \mid Y=y]\lim_{k \to \infty}\mathbb{P}_{\mathcal{D}'}[Y=y]\right) \\
&= \frac{1}{|V_j|} \sum_{y \in Y} \lim_{k \to \infty} \mathbb{P}_{\mathcal{D}'}[Y=y] \\
&= \frac{1}{|V_j|}.
\end{align*}

\noindent Thus, we have shown the marginal probability and conditional probabilities are the same at the limit.
\end{proof}

While \Cref{thm:unif_xs} provides a theoretical guarantee for the feature-label independence in the asymptotic limit, correlation between the features and the label may be more efficiently reduced via the following iterative strategy:
\begin{enumerate}
\item Train a model on the poisoned dataset generated so far.
\item Use the model's predictions on the training data to identify the most important adversarial samples and add them to the poisoned dataset.
\item Retrain the model on the augmented dataset.
\item Repeat until the poisoning budget is exhausted.
\end{enumerate}

The key idea of this strategy is to select feature vectors from the subset of the training data where the behavior is contrary to our objective (\Cref{lem:max_unfair}), $\mathcal{D}_{S=s, \hat{Y}=1-s}$, the set of samples from the protected group $s$ where the predicted outcome does not match the group. By injecting the samples into the training data with the label $Y=s$, the empirical probability $\mathbb{P}[X_i=x_i \mid Y=1-s]$ remains unchanged while $\mathbb{P}[X_i=x_i \mid Y=s]$ increases.

This strategy relies on the assumption that the majority of the samples in $\mathcal{D}_{S=s, \hat{Y}=1-s}$ is due to the feature vector $\mathbf{x}$ being is more correlated to $Y=1-s$ than $Y=s$. For these samples, this implies that the conditional probabilities are, on average, biased toward the undesired label: $\mathbb{P}[X_i=x_i \mid Y=s] < \mathbb{P}[X_i=x_i \mid Y=1-s]$. Adding those samples into the training data with the label $Y=s$ results in the probabilities become closer, reducing the correlation.

Furthermore, this new strategy directly nudges the model towards the desired unfair behavior. By adding poisoned samples that target samples that have outcomes contrary to our desired behavior, we increase $score_{NB}(\mathbf{x}, s, s)$ while reduce $score_{NB}(\mathbf{x}, s, 1-s)$ (\Cref{lem:poison_nb}).

\begin{lemma}[Impact of Sample Injection on Naive Bayes Scores]
Adding a sample $(\mathbf{x}, s, y)$ results in  an increase in $score_{NB}(\mathbf{x}, s, y)$ and a reduction in $score_{NB}(\mathbf{x}, s, 1-y)$. 
\label{lem:poison_nb}
\end{lemma}
\begin{proof}
To prove this, we show that adding $(\mathbf{x}, s, y)$ increases the value of each term $score_{NB}(\mathbf{x}, s, y)$ while it reduces or does not effect each term in $score_{NB}(\mathbf{x}, s, 1-y)$.

\noindent Let the original dataset be $\mathcal{D}$ and the new dataset be $\mathcal{D}'=\mathcal{D} \bigcup \{(\mathbf{x},S=s,Y=s)\}$. The probabilities associated with each dataset will have $\mathcal{D}$ or $\mathcal{D}'$ subscripts respectively.

\begin{itemize}
    \item \textbf{Change to $score_{NB}(\mathbf{x}, s, y)$:} The empirical prior $\mathbb{P}_{\mathcal{D}'}[Y=y]$ increases as the new sample has $Y=y$. Similarly, the empirical conditional $\mathbb{P}_{\mathcal{D}'}[S=s \mid Y=y]$ also goes up since the new sample has $S=s$ and $Y=y$. The empirical likelihoods $\mathbb{P}_{\mathcal{D}'}[X_i=x_i \mid Y=y]$ for each feature $x_i$ in $\mathbf{x}$ also increase due to the new sample. These changes collectively increase the value of $\text{score}_{NB}(\mathbf{x}, s, y)$.
    \item \textbf{Change to $score_{NB}(\mathbf{x}, s, 1-y)$:} The empirical prior $\mathbb{P}_{\mathcal{D}'}[Y=1-y]$ decreases as the total number of samples in the denominator increases, while the count of samples where $Y=1-y$ remains the same. All other terms in the score are conditional on $Y=1=y$. Since no samples are added where $Y=1-y$, these terms are unaffected. These changes collectively decrease the value of $\text{score}_{NB}(\mathbf{x}, s, 1-y)$.
\end{itemize}
\end{proof} 


This strategy of utilizing the prediction of a surrogate model is more effective than the uniform, blind injection strategy. Since the samples are drawn from the true data distribution, they represent the types of inputs the model is likely to encounter in the real world. Furthermore, by prioritizing the reduction of the correlation of realistic samples, this strategy is able to achieve a strong attack with relatively few samples in practice compared to the strategy outlined in \Cref{thm:unif_xs}.

Using the intuition provided in \Cref{thm:nb_unfair}, we present a general strategy to generate poisoned samples. Given a sensitive attribute group $s \in S$ we want to add to, we set the sensitive attribute group $S$ and label $Y$ to be $s$ (Lines \ref{step:sample}-\ref{step:flip}) and generate the input feature vector $\mathbf{x}_s = (x_{s,1}, \dots, x_{s,K})$ by selecting $\mathbf{x}$ from $\mathcal{D}_{S=s, \hat{Y}=1-s}$ (Lines \ref{step:surrogate}-\ref{step:sampleend}). If $\mathcal{D}_{S=s, \hat{Y}=1-s} = \emptyset$, then $\mathbf{x}_s$ is generated by uniformly selecting a value from $|V_j|$ for feature $x_{s,j}$ (Lines \ref{step:unifstart}-\ref{step:unifend}).

\Cref{thm:unif_xs} shows this results in poisoned samples of the form $(\mathbf{x}_s, S=s, Y=s)$, where the feature vector $\mathbf{x}_s = (x_{s,1}, \dots, x_{s,K})$ is constructed so that, as the number of added samples approaches infinity, the feature distributions become uncorrelated with the label.

\subsection{Batch Group Selection}
So far, we have outlined how to generate poisoned samples given a sensitive attribute group. In this section, we outline how to select the poisoned samples' sensitive attribute group, motivated by an analysis of how adding specific samples affects SPD disparity in a naive Bayes classifier, under the simplifying assumption that non-sensitive features are independent of the label.

\begin{table*}[]
\resizebox{1.0\textwidth}{!}{
\begin{tabular}{@{}crlrrrrlrrrrlrrrr@{}}
\toprule
& & & \multicolumn{4}{c}{\textbf{German}} & & \multicolumn{4}{c}{\textbf{Drug}} & & \multicolumn{4}{c}{\textbf{COMPAS}} \\
\textbf{Metric}       & \textbf{Method} &  & \textbf{GNB}     & \textbf{LR}    & \textbf{DT}    & \textbf{KNN}     &  & \textbf{GNB}     & \textbf{LR}    & \textbf{DT}     & \textbf{KNN}     &  & \textbf{GNB}     & \textbf{LR}    & \textbf{DT}     & \textbf{KNN}     \\
\cmidrule[0.5pt](lr){0-0} \cmidrule[0.5pt](lr){2-2} \cmidrule[0.5pt](lr){4-7} \cmidrule[0.5pt](lr){9-12}  \cmidrule[0.5pt](lr){14-17}

\multirow{5}{*}{\textbf{$\Delta$ Acc}} &  \textbf{RAA-P}  &  & \textbf{-0.030} & \textcolor{lightgray}{0.025} & \textcolor{lightgray}{-0.015} & \textcolor{lightgray}{-0.005} &  & \textbf{-0.031} & \textbf{-0.057} & \textcolor{lightgray}{-0.057} & \textbf{-0.070} &  & \textbf{-0.015} & \textbf{-0.003} & \textcolor{lightgray}{-0.013} & \underline{-0.053} \\ 
& \textbf{NRAA-P} &  & \textcolor{lightgray}{-0.105} & \textbf{0.010} & \textbf{-0.070} & \textcolor{lightgray}{-0.020} &  & \underline{-0.075} & \underline{-0.088} & \underline{-0.135} & \underline{-0.078} &  & \underline{-0.032} & \underline{-0.028} & \textcolor{lightgray}{-0.019} & \underline{-0.053}   \\
& \textbf{RAA-F}  &  & \underline{-0.065} & \underline{-0.065} & \underline{-0.100} & \textcolor{lightgray}{-0.015} &  & -0.112 & -0.145 & \textbf{-0.132} & \textbf{-0.070} &  & -0.065 & -0.134 & \textcolor{lightgray}{-0.017} & -0.055  \\
& \textbf{NRAA-F} &  & \textcolor{lightgray}{-0.105} & \textbf{0.010} & \textbf{-0.070} & \textcolor{lightgray}{-0.020} &  & \underline{-0.075} & \underline{-0.088} & \underline{-0.135} & \underline{-0.078} &  & \underline{-0.032} & \underline{-0.028} & \textcolor{lightgray}{-0.019} & \underline{-0.053} \\ 
& \textbf{Ours} &  &  -0.410 & -0.235 & -0.215 & \textbf{\underline{-0.185}} &  & -0.405 & -0.244 & -0.205 & -0.249 &  & -0.419 & -0.281 & \textbf{\underline{-0.306}} & \textbf{-0.044} \\ 

\midrule

\multirow{5}{*}{\textbf{$\Delta$ SPD}} & \textbf{RAA-P}  &  & -0.023 & \textcolor{lightgray}{-0.013} & \textcolor{lightgray}{-0.142} & \textcolor{lightgray}{-0.108} &  & 0.146 & 0.184 & \textcolor{lightgray}{-0.106} & 0.025 &  & -0.032 & 0.155 & \textcolor{lightgray}{-0.072} & 0.017 \\
& \textbf{NRAA-P} &  & \textcolor{lightgray}{-0.023} & 0.030 & -0.055 & \textcolor{lightgray}{-0.127} &  & 0.285 & 0.218 & -0.041 & 0.030 &  & -0.088 & 0.195 & \textcolor{lightgray}{-0.065} & 0.016 \\
& \textbf{RAA-F}  &  & \underline{0.046} & \underline{0.372} & \underline{-0.041} & \textcolor{lightgray}{-0.083} &  & \underline{0.340} & \underline{0.385} & \underline{0.027} & \underline{0.035} &  & \underline{0.259} & \underline{0.608} & \textcolor{lightgray}{-0.057} & \underline{0.044} \\
& \textbf{NRAA-F} &  & \textcolor{lightgray}{-0.023} & 0.030 & -0.055 & \textcolor{lightgray}{-0.127} &  & 0.285 & 0.218 & -0.041 & 0.030 &  & -0.088 & 0.195 & \textcolor{lightgray}{-0.065} & 0.016  \\ 
& \textbf{Ours}  &  & \textbf{0.273} & \textbf{0.782} & \textbf{0.637} & \textbf{\underline{0.152}} &  & \textbf{0.776} & \textbf{0.787} & \textbf{0.575} & \textbf{0.672} &  & \textbf{0.776} & \textbf{0.768} & \textbf{\underline{0.775}} & \textbf{0.114} \\
\midrule

\multirow{5}{*}{\textbf{$\Delta$ EOD}} & \textbf{RAA-P} &  & 0.154 & \textcolor{lightgray}{-0.018} & \textcolor{lightgray}{-0.058} & \textcolor{lightgray}{-0.116} &  & 0.200 & 0.251 & \textcolor{lightgray}{-0.003} & 0.073 &  & 0.132 & 0.336 & \textcolor{lightgray}{-0.050} & 0.034 \\
& \textbf{NRAA-P} &  & \textcolor{lightgray}{-0.047} & 0.050 & 0.174 & \textcolor{lightgray}{-0.144} &  & 0.327 & 0.270 & 0.007 & 0.064 &  & 0.046 & 0.367 & \textcolor{lightgray}{-0.043} & 0.034 \\
& \textbf{RAA-F} &  & \underline{0.306} & \underline{0.370} & \underline{0.214} & \textcolor{lightgray}{-0.010} &  & \underline{0.419} & \underline{0.439} & \underline{0.113} & \underline{0.082} &  & \underline{0.373} & \underline{0.665} & \textcolor{lightgray}{-0.036} & \underline{0.053} \\
& \textbf{NRAA-F} &  & \textcolor{lightgray}{-0.047} & 0.050 & 0.174 & \textcolor{lightgray}{-0.144} &  & 0.327 & 0.270 & 0.007 & 0.064 &  & 0.046 & 0.367 & \textcolor{lightgray}{-0.043} & 0.034 \\
& \textbf{Ours} & &  \textbf{0.453} & \textbf{0.828} & \textbf{0.765} & \textbf{\underline{0.324}} &  & \textbf{0.825} & \textbf{0.834} & \textbf{0.729} & \textbf{0.749} &  & \textbf{0.775} & \textbf{0.770} & \textbf{\underline{0.775}} & \textbf{0.159} \\

\bottomrule
\end{tabular}
}
    \caption{Performance comparison of attack methods when adding $|\mathcal{D}|$ poisoned samples to the dataset. The highest value is in \textbf{bold} and the second-highest is \underline{underlined}. Values are marked in gray and ignored if both SPD and EOD are reduced.}
    \label{tab:main}
\end{table*}

\begin{lemma}[Simplified Prediction Probability under Feature Independence]
Given a naive Bayes classifier where all non-sensitive features are independent of the label, that is, $\mathbb{P}[X_i=x_i|Y=y] = \mathbb{P}[X_i=x_i]~\forall x_i \in V_i,~y \in \{0, 1\}$, the probability of predicting the positive class ($\hat{Y}=1$) for a sample with a given sensitive attribute value $S=s$ is
$$\mathbb{P}(\hat{Y}=1 \mid S=s) = \mathbb{I}\left[score'_{NB}(s, 1) > score'_{NB}(s, 0)\right]$$
where $\mathbb{I}[\cdot]$ is the indicator function.
\label{lem:simplified_pred}
\end{lemma}
\begin{proof}
For a given sensitive attribute value $s$, this condition is either true for all samples or false for all samples, as the simplified scores are constant with respect to the non-sensitive features. Therefore, the probability of predicting $\hat{Y}=1$ for a sample with attribute $S=s$ is simply the proportion of samples in the dataset with that attribute for which the classifier predicts $\hat{Y}=1$:
$$\mathbb{P}(\hat{Y} = 1 \mid S = s) = \frac{|\mathcal{D}_{\hat{Y}=1, S=s}|}{|\mathcal{D}_{S=s}|}$$

Since the decision is the same for all samples with $S=s$, this simplifies to:
$$\frac{\sum_{\mathbf{x} \in \mathcal{D}_{S=s}} \mathbb{I}[score'_{NB}(s, 1) > score'_{NB}(s, 0)]}{|\mathcal{D}_{S=s}|}$$

The indicator function's value is constant for all samples where $S=s$, so we can take it outside the summation:
$$\frac{|\mathcal{D}_{S=s}| \cdot \mathbb{I}[score'_{NB}(s, 1) > score'_{NB}(s, 0)]}{|\mathcal{D}_{S=s}|}$$
Which simplifies to:
$$\mathbb{I}[score'_{NB}(s, 1) > score'_{NB}(s, 0)].$$
\end{proof}

SPD is defined as follows:
$$\left| \mathbb{P}(\hat{Y} = 1 \mid S = 1) - \mathbb{P}(\hat{Y} = 1 \mid S = 0) \right|.$$
\Cref{lem:simplified_pred} shows that in our simplified scenario, $\mathbb{P}(\hat{Y} = 1 \mid S = 1)$ can only take on binary values, posing a challenge for analysis. To overcome this, we approximate this with a continuous score: the margin of preference. We define the margin of preference for the positive class for group $s$ as the difference between the simplified scores:
\begin{align*}
  M_s &= score'_{NB}(s, 1) - score'_{NB}(s, 0) \\
  &= \mathbb{P}[Y=1] \mathbb{P}[S=s \mid Y=1] - \mathbb{P}[Y=0] \mathbb{P}[S=s \mid Y=0].  
\end{align*}

A change in SPD occurs only when the sign of a margin flips (i.e., when $M_s$ crosses zero). Since we are interested in the sensitivity of the classifier's decisions, we propose a continuous alternative to the SPD, which we will call the Continuous Disparity Margin (CDM), defined as the difference between the margins of the two sensitive groups:
\begin{equation}
    CDM = |M_1 - M_0|.
    \label{eq:cdm}
\end{equation}

Let $CDM_{n, s}$ represent the CDM after adding $n$ samples with $S=s$ and $Y=s$. Our strategy for selecting the sensitive attribute group to add $n$ poisoned samples is to first calculate $CDM_{n, 1}$ and $CDM_{n, 0}$, then add the samples to the group corresponding with the higher value.

\subsection{Dataset Selection}
The final poisoned dataset adopted is the one that that achieves the best score on a separate hold-out test set (Line \ref{step:sel}). The candidate dataset's score is determined by considering its impact on both performance and fairness. $\mathcal{F}_{perf}$ denotes the set of functions to evaluate performance metrics on the model output. $\mathcal{F}_{fair}$ is the set of functions to evaluate fairness metrics. Each function maps model predictions $\hat{Y}$ and true labels $Y$ to a value in the range $[0, 1]$, where higher values favor the attacker (better model performance or greater fairness disparity between groups). We define the aggregated output of each set as the mean of its constituent functions, denoted $\overline{\mathcal{F}_{perf}}$ and $\overline{\mathcal{F}_{fair}}$ respectively.

The goal of the candidate dataset score is to assign a higher value to a candidate dataset that balances the trade-off between the performance and fairness of the model. However, in practice, the trade-off is not static. The relative importance of performance and fairness may change depending on the scenario:
\begin{enumerate}
    \item If the utility of the compromised model is worse than that of the clean model, ($m_{clean}$), the victim may not deployed the model. The attacker may want to avoid this situation by trading off attack effectiveness for performance.
    \item Meanwhile, if the performance is equivalent or better than the performance of the clean model, the above concern is alleviated and the attacker may want to focus on the effectiveness of the attack.
\end{enumerate}

\begin{figure*}[t]
    \centering
    \begin{subfigure}[b]{0.48\linewidth}
        \includegraphics[trim={0 0.4cm 0 0},clip,width=\linewidth]{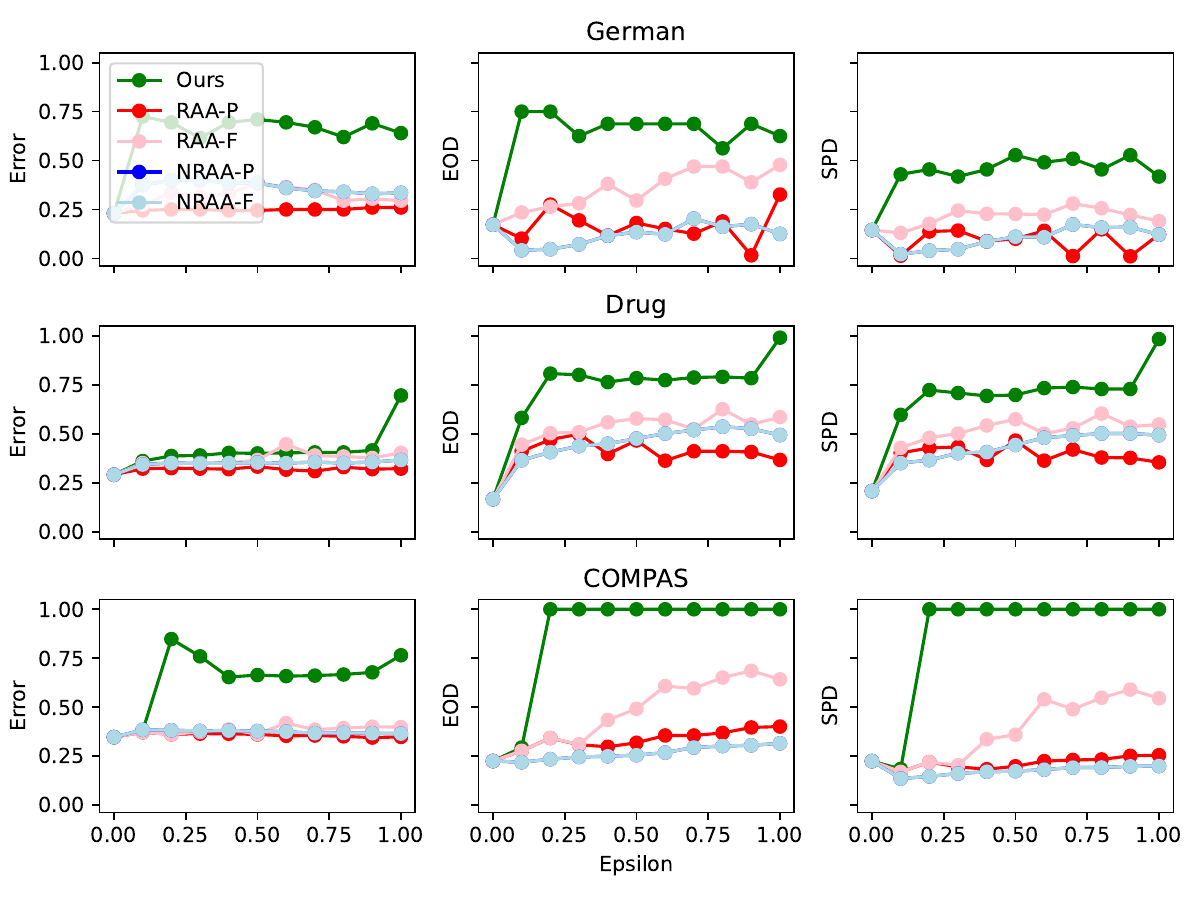}
        \caption{Gaussian Naive Bayes} \label{fig:gnb}
    \end{subfigure}%
    \hfill
    \begin{subfigure}[b]{0.48\linewidth}
        \includegraphics[trim={0 0.4cm 0 0},clip,width=\linewidth]{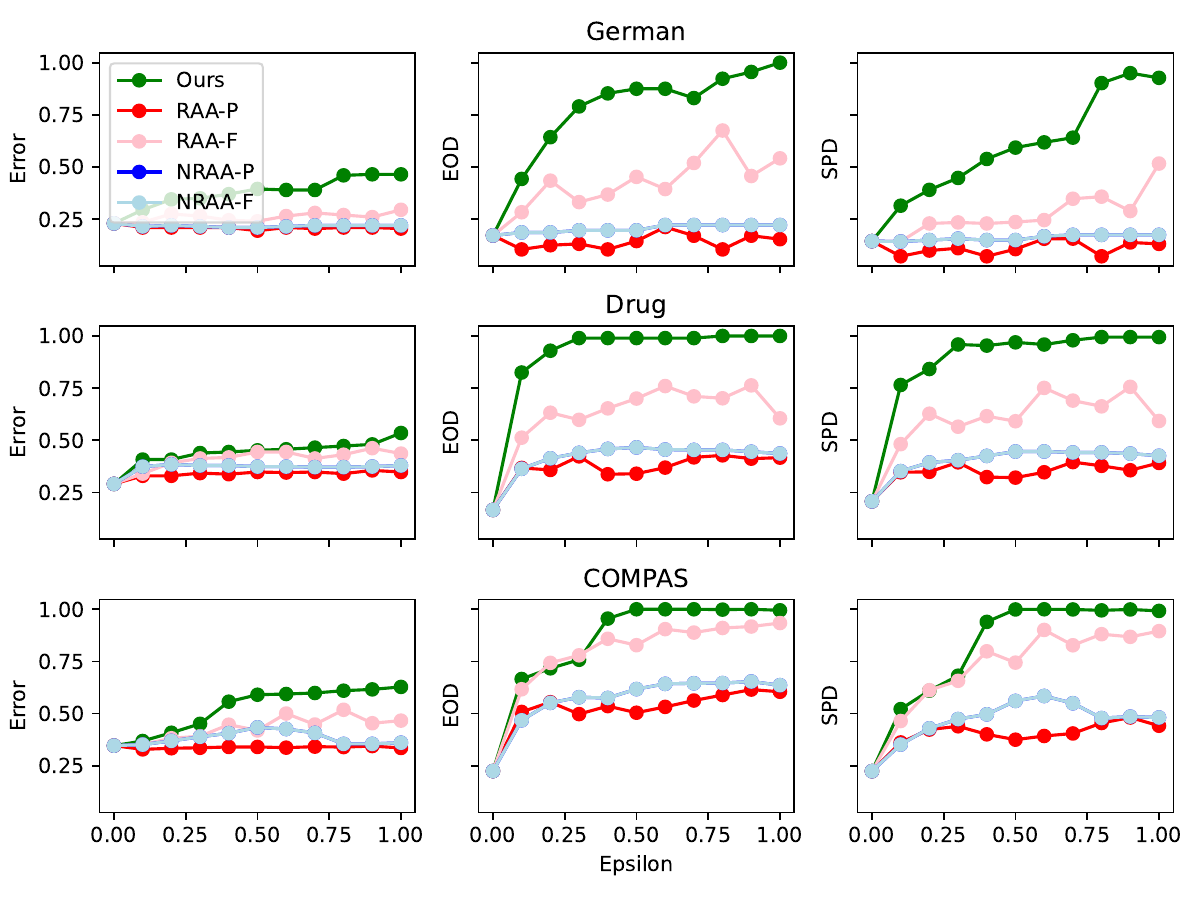}
        \caption{Logistic Regression} \label{fig:lr}
    \end{subfigure}%
    \hfill
    \begin{subfigure}[b]{0.48\linewidth}
        \includegraphics[trim={0 0.4cm 0 0},clip,width=\linewidth]{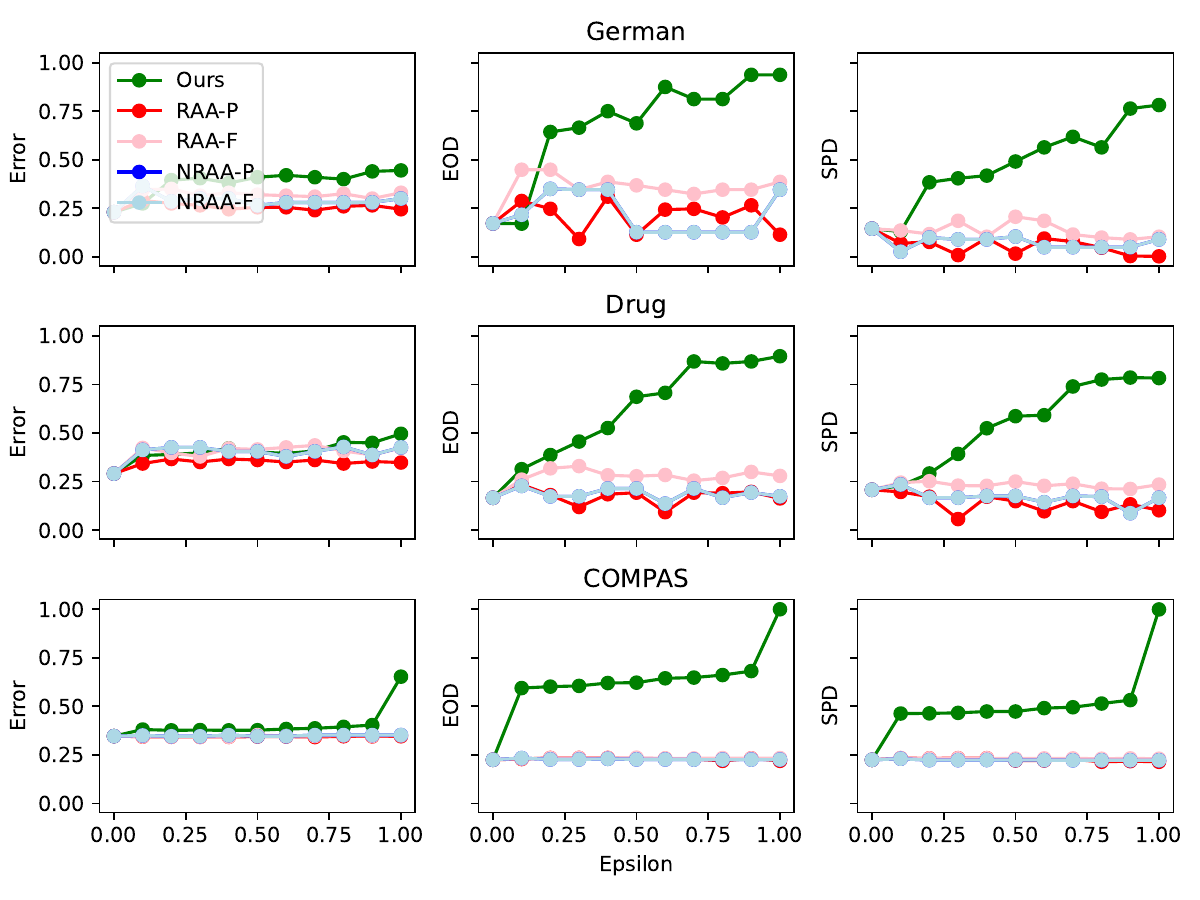}
        \caption{Decision Tree} \label{fig:dt}
    \end{subfigure}%
    \hfill
    \begin{subfigure}[b]{0.48\linewidth}
        \includegraphics[trim={0 0.4cm 0 0},clip,width=\linewidth]{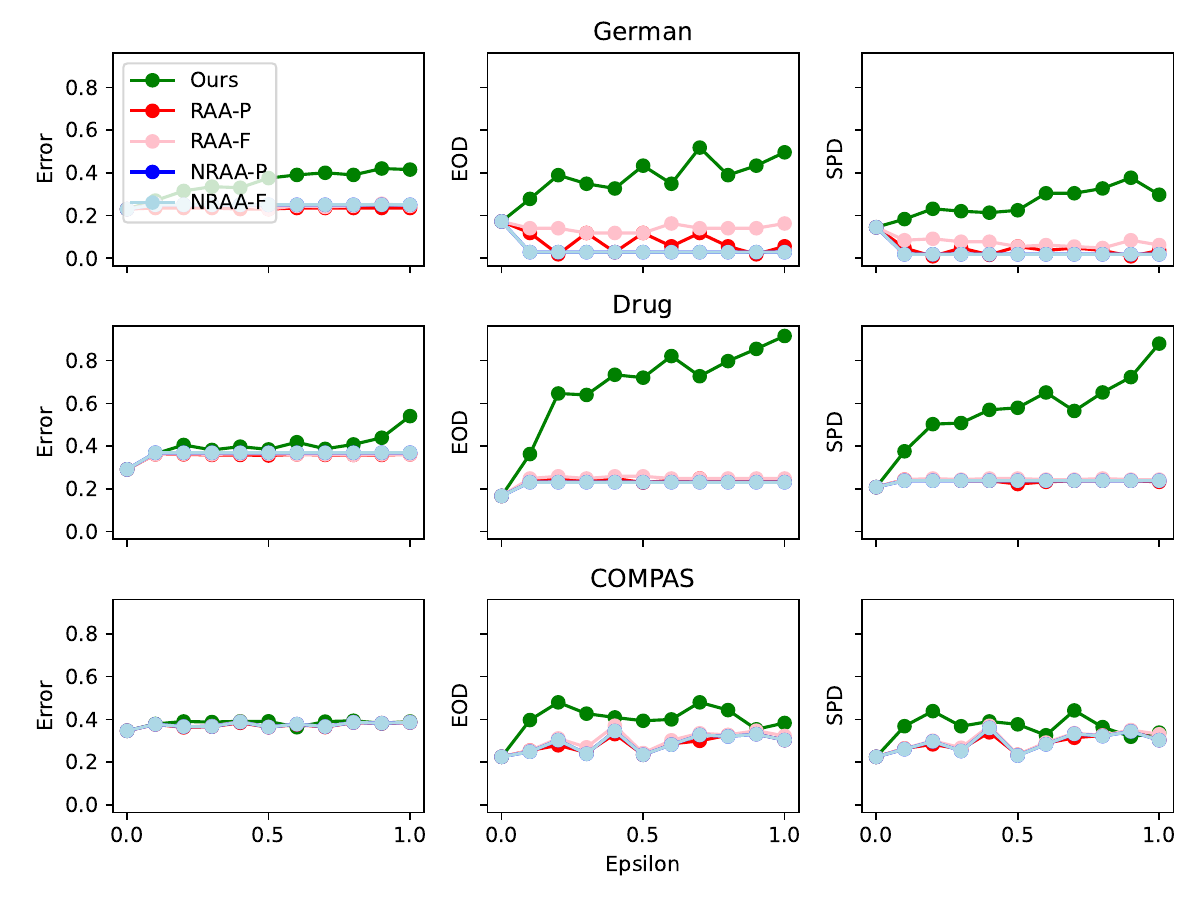}
        \caption{K-Nearest Neighbor} \label{fig:knn}
    \end{subfigure}%
    \caption{Our method consistently achieves the most significant degradation fairness compared to other attack methods against the base models over a varying number of poisoned samples.} \label{fig:all_models}
\end{figure*}
\raggedbottom

Let the percentage change of a metric be notated as
$$f_{\Delta\mathcal{F}}(\hat{Y}, Y) := \frac{\overline{\mathcal{F}}(\hat{Y}, Y) - \overline{\mathcal{F}}(m_{clean}(\mathcal{D}), Y)}{\overline{\mathcal{F}}(m_{clean}(\mathcal{D}), Y)}$$
The candidate dataset score is defined as follows:
\begin{align}
\label{eq:selscore}
f(\hat{Y}, Y) =&~f_{\Delta\mathcal{F}_{fair}}(\hat{Y}, Y) + \alpha f_{\Delta\mathcal{F}_{perf}}(\hat{Y}, Y)~+ \nonumber \\
&\beta \cdot \mathbb{I}\left[f_{\Delta\mathcal{F}_{perf}}(\hat{Y}, Y) < 0\right] \cdot f_{\Delta\mathcal{F}_{perf}}(\hat{Y}, Y)
\end{align}
where $\alpha \geq 0$ controls the fairness-performance trade-off and $\beta \geq 0$ controls the additional weight applied when performance is negatively impacted by the attack (as indicated by the condition $\mathbb{I}[f_{\Delta\mathcal{F}_{perf}}(\hat{Y}, Y) < 0]$).

This scoring function $f$ aligns with our goals. When the performance is equivalent or better than the performance of the clean model, the fairness-performance trade-off is controlled by $\alpha$. However, when the performance is worse than the clean model's (case 2), the score $f$ is $f_{\Delta\mathcal{F}_{fair}}(\hat{Y}, Y) + (\alpha + \beta) f_{\Delta\mathcal{F}_{perf}}(\hat{Y}, Y)$. Since $\alpha, \beta \geq 0$, the performance term may be upweighed, reducing the relative influence of the fairness term.

\section{Experiment}
\begin{table}[]
\resizebox{0.5\textwidth}{!}{
\begin{tabular}{@{}crccc@{}}
\toprule
&& \textbf{German} & \textbf{Drug} & \textbf{COMPAS} \\ \cmidrule[0.5pt](lr){3-3} \cmidrule[0.5pt](lr){4-4} \cmidrule[0.5pt](l){5-5}
\textbf{\# Features} && 58 & 13 & 8 \\ \midrule
\multirow{3}{*}{\textbf{\# Samples}} & \textbf{Overall} & 800 & 1500 & 5771 \\
& \textbf{S=1} & 255 & 753 & 1095\\
& \textbf{S=0} & 545 & 747 & 4676\\ \midrule
\multirow{3}{*}{\textbf{\% Y=1}} & \textbf{Overall} & 29.87\% & 55.07\% & 55.05\% \\
& \textbf{S=1} & 36.47\% & 63.88\% & 65.48\% \\
& \textbf{S=0} & 26.79\% & 46.19\% & 52.61\% \\

\bottomrule
\end{tabular}
}

    \caption{Dataset statistics}
    \label{tab:data_stats}
\end{table}
\subsection{Datasets}
We use three widely used datasets in the field of algorithm fairness: the \textbf{German} credit dataset \cite{misc_statlog_(german_credit_data)_144}, the \textbf{drug} use prediction dataset \cite{drug}, and the \textbf{COMPAS} recividism dataset \cite{compas}.

We evaluate each method using {\em gender} as the sensitive attribute. Each dataset is normalized and divided into an 80-20 train-test split. The German dataset is the smallest, followed by the drug dataset, and the COMPAS dataset is the largest. 
\Cref{tab:data_stats} is a table of the statistics associated with each dataset used in the evaluations.
\hide{
\paragraph{German Dataset} The German prediction task \cite{misc_statlog_(german_credit_data)_144} is to leverage individual financial and banking details of an individual to classify whether the individual is a good or bad credit risk.
\paragraph{COMPAS Dataset} The COMPAS prediction task \cite{compas} is to leverage information about an individual's criminal history and demographics to identify whether the individual would re-offend within two years.
\paragraph{Drug Dataset} The Drug prediction task \cite{drug} is to leverage information about an individual to identify whether the individual had ever consumed cocaine.
}
\subsection{Metrics}
The performance metric we evaluate is accuracy. For the fairness metric, we evaluate using two measures of group fairness: statistical parity difference (SPD) \cite{Kamiran2012} and equalized odds difference (EOD) \cite{NIPS2016_9d268236}.

SPD measures the difference in favorable outcomes between the advantaged and disadvantaged groups.
\[
\text{SPD} = \left| \mathbb{P}(\hat{Y} = 1 \mid S = 1) - \mathbb{P}(\hat{Y} = 1 \mid S = 0) \right|,
\]

whereas EOD measures the maximum difference in the true and false positive rates between the advantaged and disadvantaged groups
\begin{align*}
\text{EOD} = &\max_{y \in \{0,1\}}  \Big| \mathbb{P}(\hat{Y} = 1 \mid Y = y, S = 1) \\
& - \mathbb{P}(\hat{Y} = 1 \mid Y = y, S = 0) \Big|.
\end{align*}

\begin{figure*}[]
    \centering
    \includegraphics[trim={1.5cm 0 2cm 1.5cm},clip,width=1\linewidth]{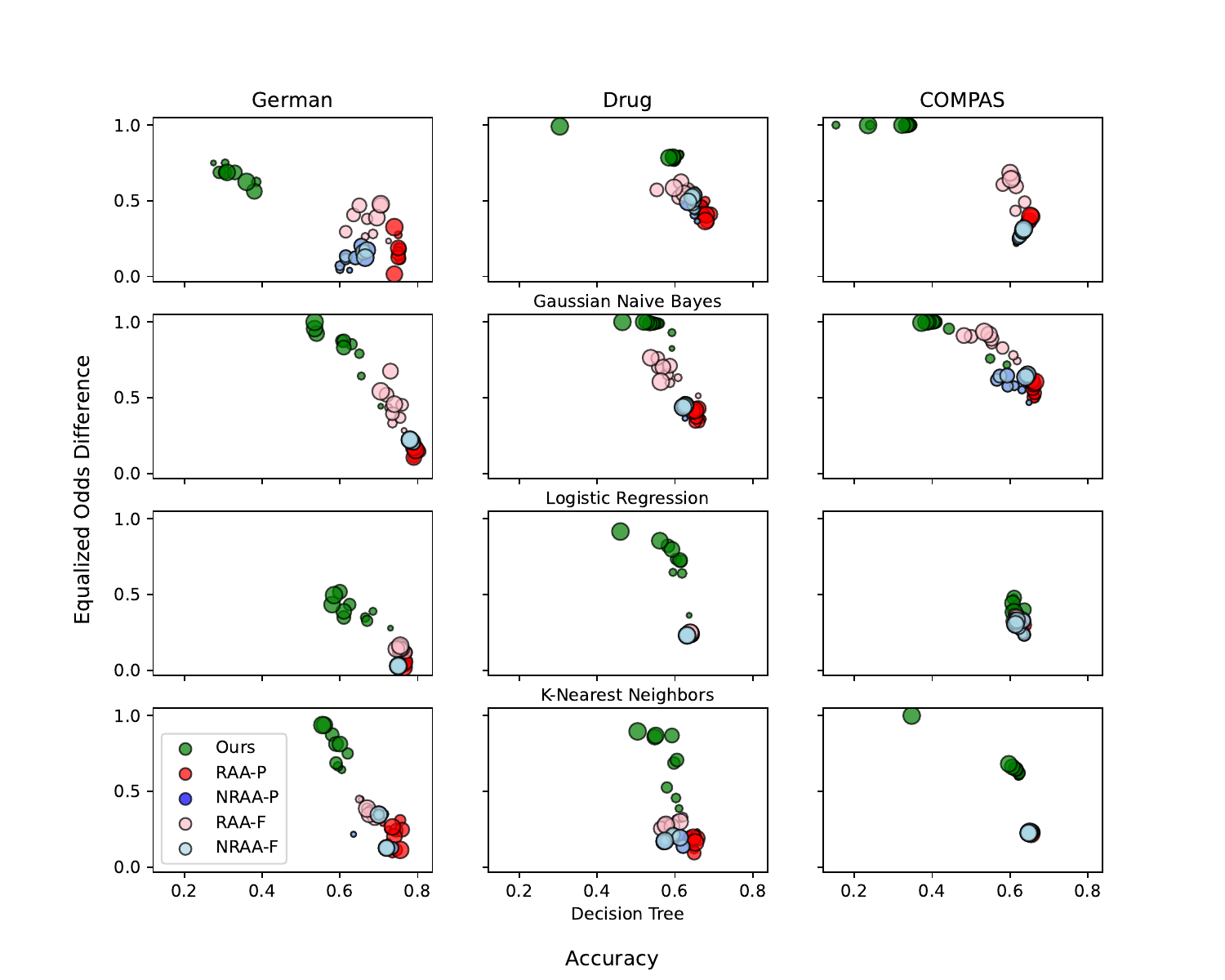}
    \caption{Our method consistently offers the best balance, achieving notably higher fairness degradation when comparing the trade-off between the fairness attack methods for each dataset and base model across all epsilon.} \label{fig:tradeoff}
\end{figure*}

\subsection{Methods}
We compare against two methods.

The first method is the {\em Random Anchoring Attack (RAA)} \cite{Mehrabi_Naveed_Morstatter_Galstyan_2021}. A positive and negative sample is selected uniformly at random from the disadvantaged and advantaged group and the label flipped. Then, this sample is projected to the feasible set. The samples are duplicated so that the overall proportions of positive and negative samples are unchanged.

The second method is the {\em Non-Random Anchoring Attack (NRAA)} \cite{Mehrabi_Naveed_Morstatter_Galstyan_2021} which differs from RAA in that instead of selecting samples uniformly at random, the selected samples are the ones with the largest number of nearby samples with the same sensitive attribute group and label.

To show the range of trade-offs achieved by each fairness attack method, we employ both a performance-based selection method and a fairness-based selection method as proposed in prior work to select the evaluation dataset \cite{Mehrabi_Naveed_Morstatter_Galstyan_2021, liu2024towards, nalmpantis2022re}. The performance-based selection method scores each dataset on the accuracy on an evaluation dataset. The fairness-based selection method assigns a score to each candidate datasets that is the sum of the SPD and EOD. Both selection methods then select the candidate dataset that has the highest score. 

To denote the selection method used, we append a suffix to the fairness attack's name: ``+P'' if the performance-based selection method is used, and ``+F'' if the fairness-based selection selection method is used.

\subsection{Results}  We compare the performance of our method against the existing methods over four machine learning classification models: Gaussian naive Bayes (GNB), logistic regression (LR), decision tree (DT), and K-nearest neighbors (KNN). We vary the proportions of adversary-injected samples (as controlled by $\varepsilon$) and apply each method to generate 100 candidate datasets. The final dataset with is selected with the specified selection criteria (the performance-based, fairness-based, or our trade-off-based selection criteria). For our dataset selection score function, we set the trade-off parameters $\alpha$ and $\beta$ of the dataset selection score function to be 0.5 to consider both performance and fairness. The effects of varying these hyperparameter are explored in the ablation.

\Cref{tab:main} shows the trade-off of each method over all combinations of datasets and models examined (visually summarized in \Cref{fig:summary}). This table compares the change in accuracy and fairness (SPD and EOD) to the performance of the clean model for $\varepsilon = 1$. Notably, while our method generally has the largest reduction in accuracy than the other methods, our method is able to increase (SPD, EOD) significantly more than the existing methods with the same poisoning budget. In fact, when $\varepsilon = 1$, our method is the only one that reaches maximum EOD value of 1 for four of the twelve settings (33\%) while the other methods never achieve maximum EOD.

We find that, compared to the other fairness attack methods, our method empirically achieves greater reduction in both fairness metrics. We investigate the relationship between the number of poisoned samples as controlled by $\varepsilon$ and the evaluation metrics for each base model: Gaussian naive Bayes (\Cref{fig:gnb}), logistic regression (\Cref{fig:lr}), decision tree (\Cref{fig:dt}), and K-nearest neighbors (\Cref{fig:knn}). At most values of $\varepsilon$, our method is generally able to achieve a significantly stronger effect on the fairness measures compared to the other fairness attack methods. 


For the logistic regression and Gaussian naive Bayes models, \name\ is generally competitive to RAA-F with greater unfairness at the expense of a higher rate of error. For the decision tree and K-nearest neighbor models, \name\ is able to effectively reduce the fairness measures while the other attacks generally has negligible to positive effect on those measures.



\Cref{fig:tradeoff} illustrates the trade-off between accuracy and fairness (EOD) for each dataset. Across the model architectures and the datasets, \name\ consistently demonstrate a trend of maintaining comparable trade-offs to the other fairness attack strategies while achieving significantly higher EOD values. For many of the settings, \name\ generally maintains a comparable accuracy range while achieving significantly higher EOD values, thus demonstrating a clear and substantial advantage.


\begin{figure}
    \centering
    \includegraphics[trim={0 0 0 0},clip,width=1\linewidth]{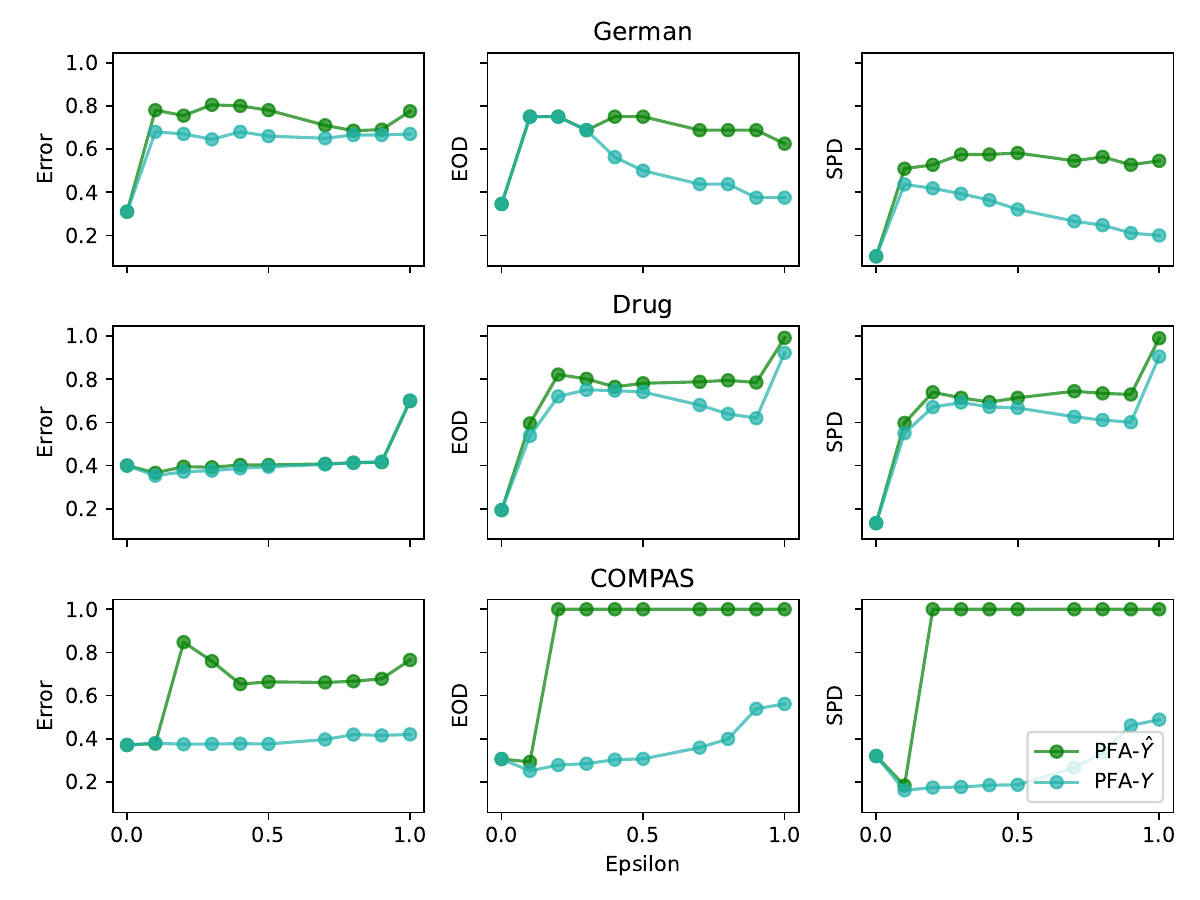}
    \caption{Comparison of \name\ using $\hat{Y}$ versus $Y$ against a Gaussian naive Bayes model over a varying number of poisoned samples. Using $\hat{Y}$ generally results in a stronger attack.} \label{fig:yyhat}
\end{figure}

\begin{figure}
    \centering
    \includegraphics[trim={3cm 2cm 3cm 2.5cm},clip,width=1\linewidth]{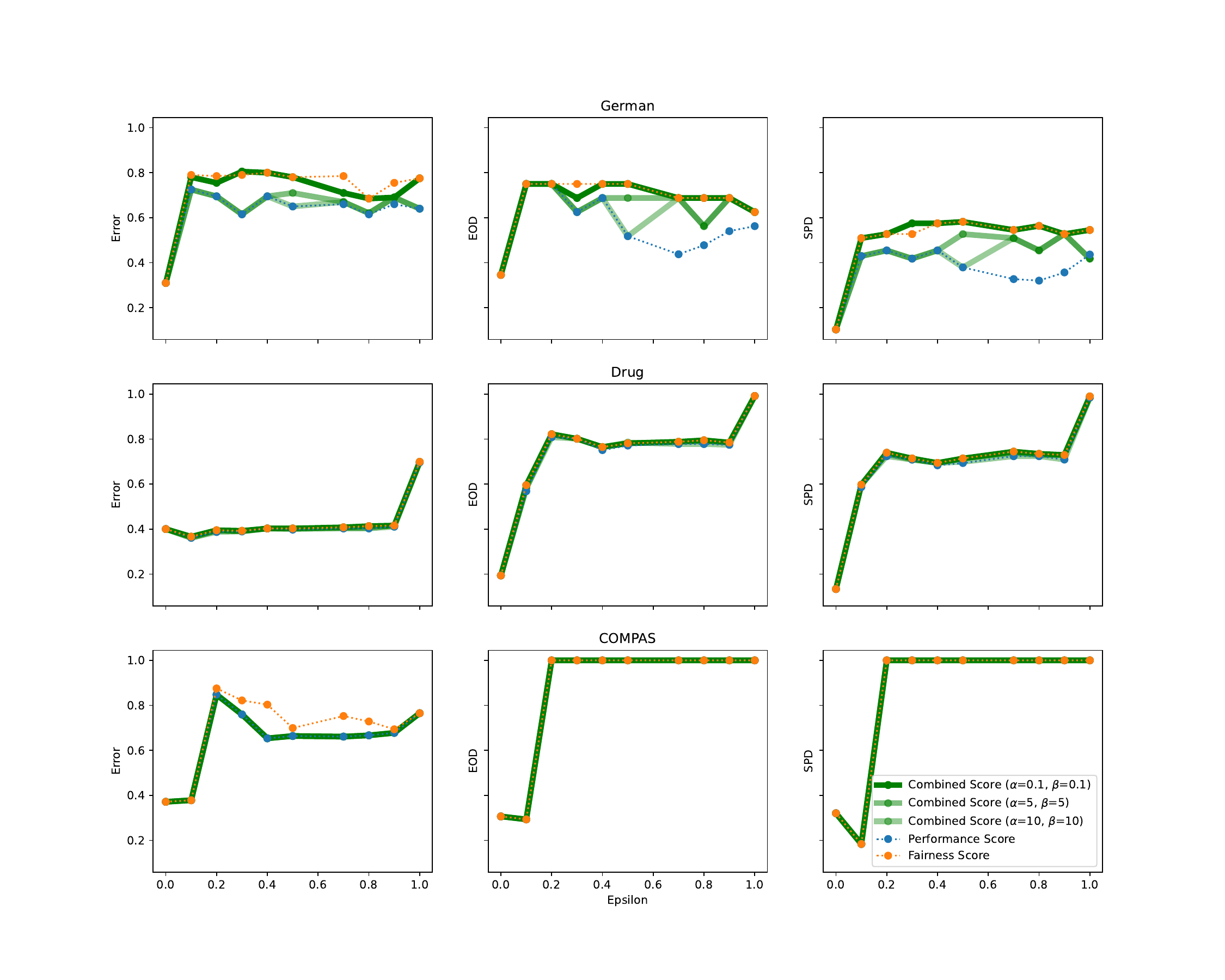}
    \caption{Comparison of candidate dataset selection methods against a Gaussian naive Bayes model over a varying number of poisoned samples. Our selection method allows users to select a dataset with their desired trade-offs.} \label{fig:gnbsel}
\end{figure}

\subsection{Ablation}
\name\ has three main components: (1) sampling from a subset controlled by $\hat{Y}$, (2) using the subset sampling heuristic rather than the uniform sample generation strategy, and (3) varying the proportion sampled from each subset based on the intermediary results of a surrogate model. 

We first investigate the first aspect: the use of $\hat{Y}$ versus $Y$. Results using the Gaussian naive Bayes model is shown in \Cref{fig:yyhat}
(results for all models in Appendix \ref{app:fullmodels}). Using $\hat{Y}$ generally results in higher error and unfairness, resulting in a stronger attack with a good trade-off. On the German and Drug datasets, errors are similar while on the COMPAS dataset, error is significantly higher, but so is the SPD and EOD values. For small values of $\varepsilon$, the performance of \name\ using $\hat{Y}$ and $Y$ are similar, but as $\varepsilon$ grows, so does the advantage of using $\hat{Y}$ as both SPD and EOD is higher with that variation.

For poisoned sample generation, we compare \name\'s method of using a heuristic of selecting real samples and setting the label to match the sensitive attribute group value versus generating feasible samples (as presented in \Cref{thm:unif_xs}). Results using the Gaussian naive Bayes model is shown in \Cref{fig:poison}
(results for all models in Appendix \ref{app:fullmodels}). Generally, the trends are similar between the two variants. The heuristic-based generation strategy generally has similar or higher error than the feasible generation strategy but much higher unfairness (EOD and SPD values), resulting in better tradeoffs.

\begin{figure}
    \centering
    \includegraphics[trim={0 0 0 0},clip,width=1\linewidth]{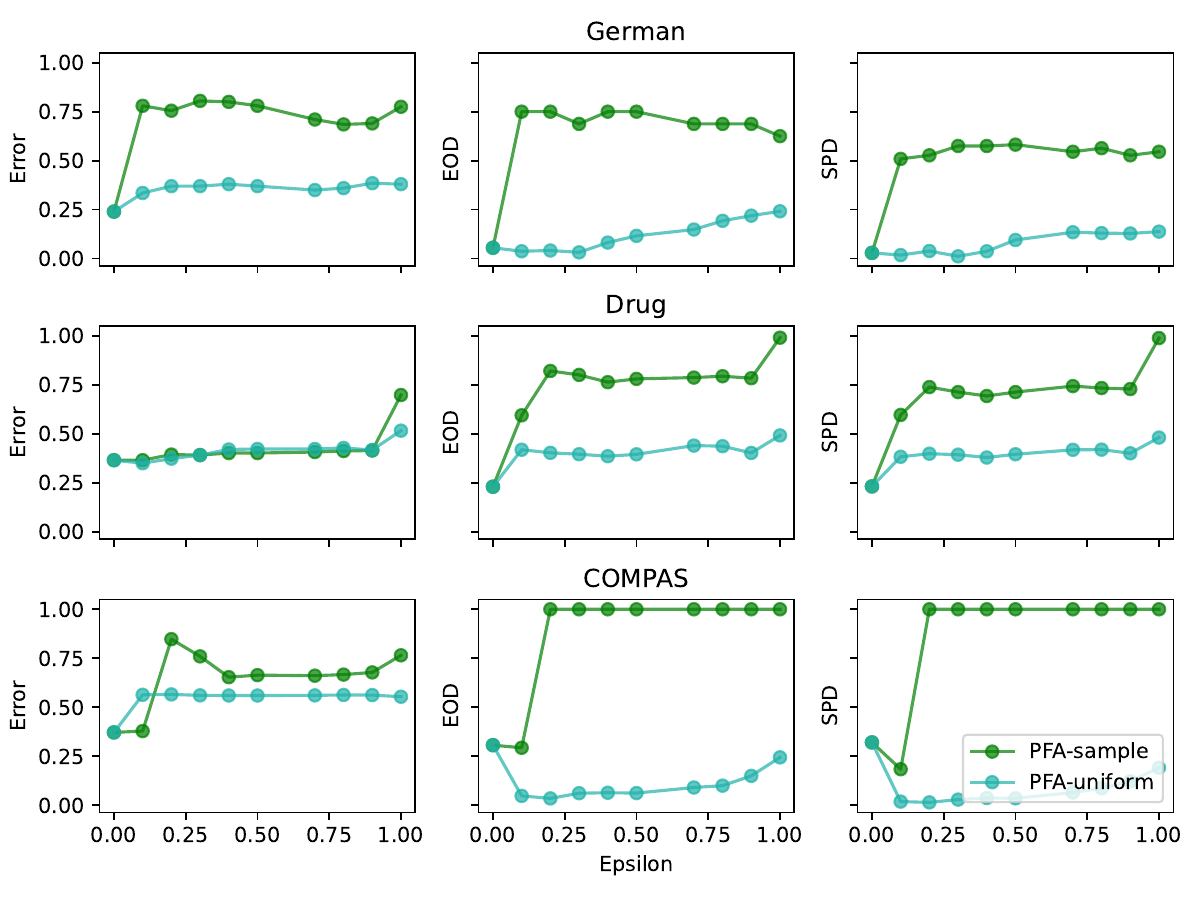}
    \caption{Comparison of \name\ generating poisoned samples with subset sampling versus feasible sample generation against a Gaussian naive Bayes model over a varying number of poisoned samples. Using real samples for the poisoned samples generally results in a stronger attack.} \label{fig:poison}
\end{figure}

The final aspect of \name\ is the dataset selection strategy, where we introduce a new candidate dataset scoring method \Cref{eq:selscore} to control the trade-off between performance and fairness. We compare the impact of the dataset selection criteria on the evaluation metrics for a Gaussian naive Bayes model for each dataset in \Cref{fig:gnbsel}. The figures for the remaining base models can be found in Appendix \ref{app:fullmodels}.

We compare the effect of our selection method with the selection methods introduced in prior work: the performance selection method \cite{Mehrabi_Naveed_Morstatter_Galstyan_2021} and the fairness selection method \cite{nalmpantis2022re}. When the parameters ($\alpha,~\beta$) are small (0.1), the fairness values of the model selected by our selection method exactly matches those selected by the fairness selection method while larger values of $\alpha$ and $\beta$ results in the metrics of the selected samples getting closer to those of the performance selection method. The comparison on the COMPAS dataset shows an advantage of the combined method over just performance or fairness. Since there are candidate models with the same high fairness values but varying error rates, the combined method with ($\alpha,~\beta = 0.1$) is able to select the candidate with a high fairness value matching the candidate selected with the fairness selection method while having a low error rate that matches the candidate selected by the performance selection method.

These results show that our selection strategy provides a controllable mechanism for users to tune their preference for the trade-off between the extent of performance degradation and the strength of the fairness attack. This tunability allows for a more nuanced and user-driven approach to generating poisoned datasets.

\section{Conclusion}
Our work highlights a novel and stealthy method that can degrade equitable outcomes while maintaining overall performance for any classifier, highlighting the vulnerability of fairness in real-world deployment. The key idea is to iteratively generate poisoned datasets using a surrogate model to guide targeted sample additions based on a principled decision rule maximizing disparity in group-conditional probabilities. Through extensive benchmarking on four base models, we demonstrate that our method is significantly more effective at attacking fairness compared to existing approaches.

These findings underscore the importance of actively monitoring fairness metrics and motivate future research to develop robust defenses against such attacks to ensure that machine learning models remain fair and trustworthy in practice.
Future work can explore robustness against a variety of different settings such as against fairness-enhanced models, a surrogate model mismatch, different modalities, and extend theoretical analysis to other model families.

\bibliographystyle{IEEEtran}
\bibliography{citations}
\begin{appendices}

\section{Results Across All Models}\label{app:fullmodels}

\Cref{fig:sel_score_abl} shows the impact of each of the three dataset selection criteria applied to our method on the evaluation metrics for each of the base models.

\begin{figure*}[t]
    \centering
    \begin{subfigure}[b]{0.48\linewidth}
        \includegraphics[trim={3cm 0 3cm 0},clip,width=\linewidth]{sections/figures/GNB_score.pdf}
        \caption{Gaussian Naive Bayes} 
    \end{subfigure}
    \hfill
    \begin{subfigure}[b]{0.48\linewidth}
        \includegraphics[trim={3cm 0 3cm 0},clip,width=\linewidth]{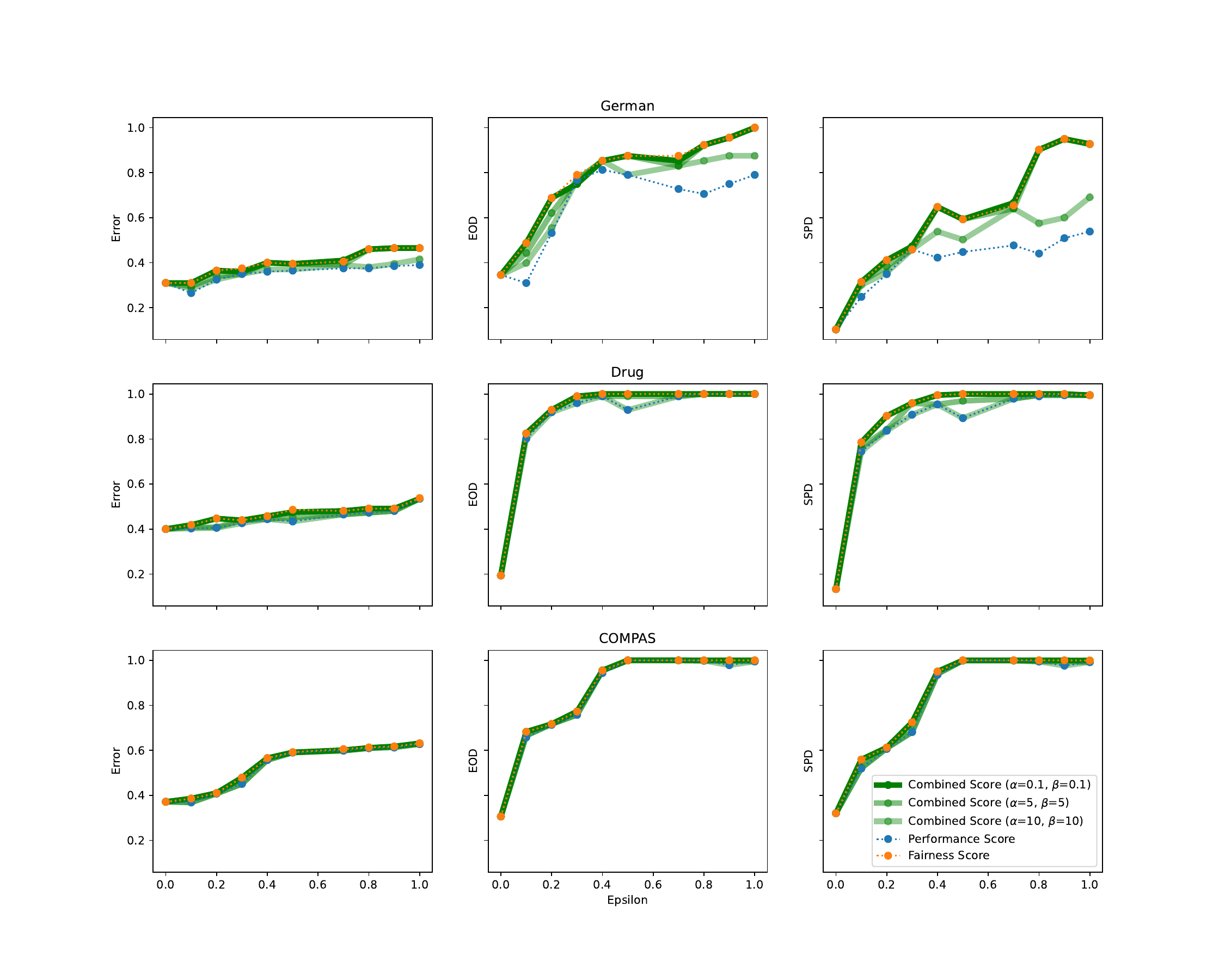}
        \caption{Logistic Regression} 
    \end{subfigure}
    
    \vfill
    
    \begin{subfigure}[b]{0.48\linewidth}
        \includegraphics[trim={3cm 0 3cm 0},clip,width=\linewidth]{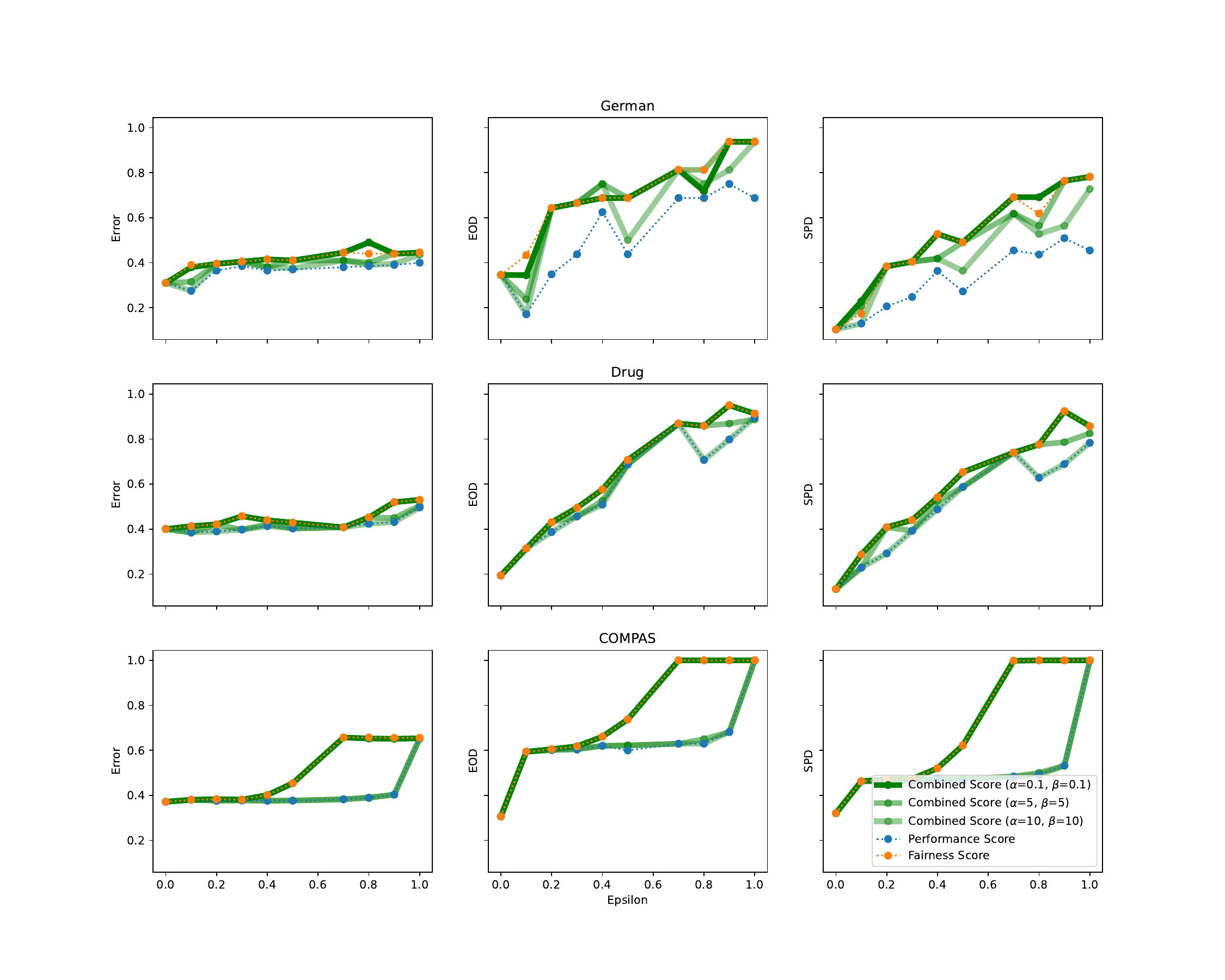}
        \caption{Decision Tree}
    \end{subfigure}
    \hfill
    \begin{subfigure}[b]{0.48\linewidth}
        \includegraphics[trim={3cm 0 3cm 0},clip,width=\linewidth]{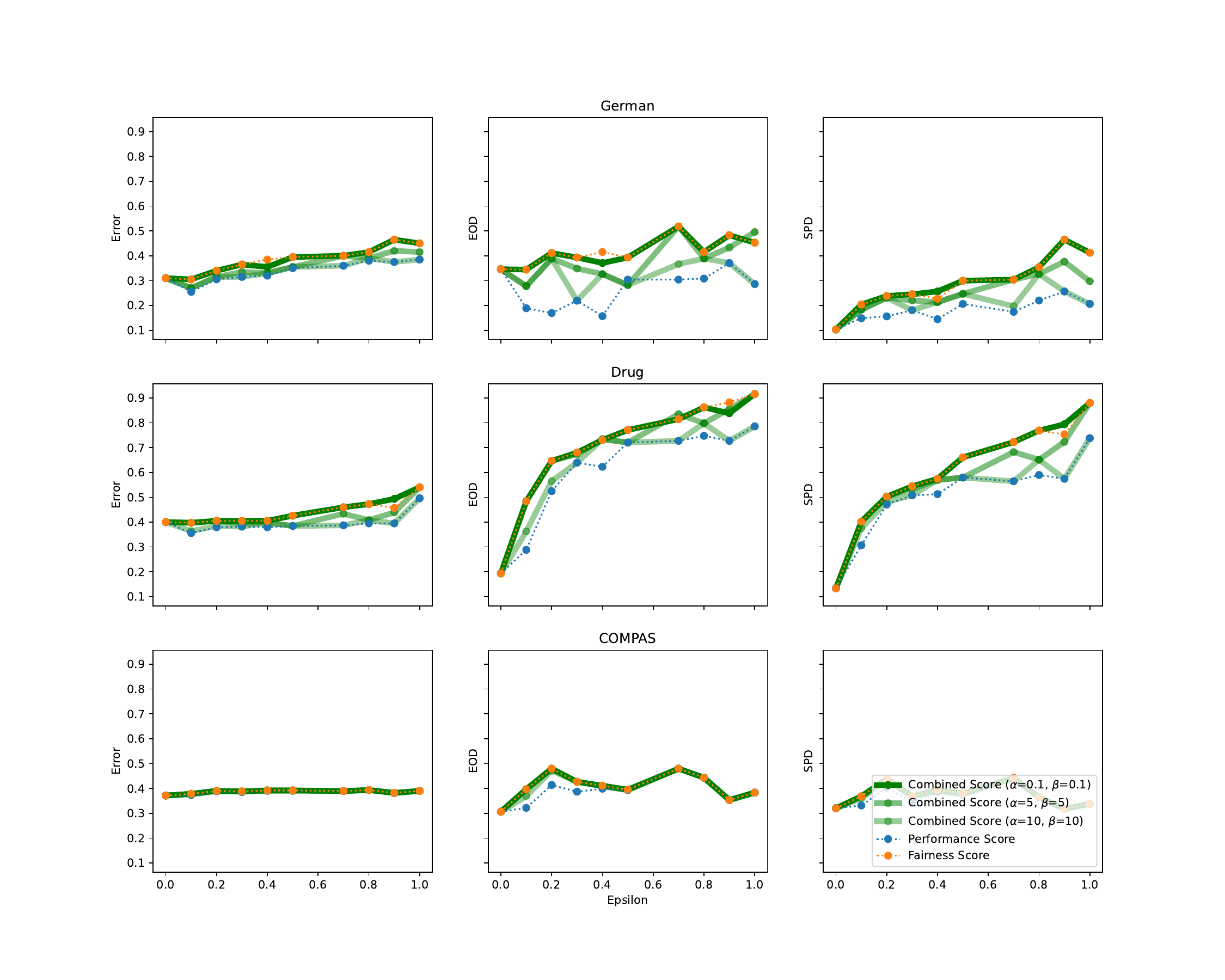}
        \caption{K-Nearest Neighbors}
    \end{subfigure}
    \caption{Comparison of the dataset selection strategies against the base models over a varying number of poisoned samples.} \label{fig:sel_score_abl}
\end{figure*}
\raggedbottom

\Cref{fig:all_models_abl} shows the impact of using $\hat{Y}$ versus $Y$ in \name\ across each of the three dataset selection criteria applied to our method for each of the base models.

\begin{figure*}[t]
    \centering
    \begin{subfigure}[b]{0.48\linewidth}
        \includegraphics[trim={0 0 0 0},clip,width=\linewidth]{sections/figures/GNB_abl.pdf}
        \caption{Gaussian Naive Bayes} \label{fig:gnbabl}
    \end{subfigure}
    \hfill
    \begin{subfigure}[b]{0.48\linewidth}
        \includegraphics[trim={0 0 0 0},clip,width=\linewidth]{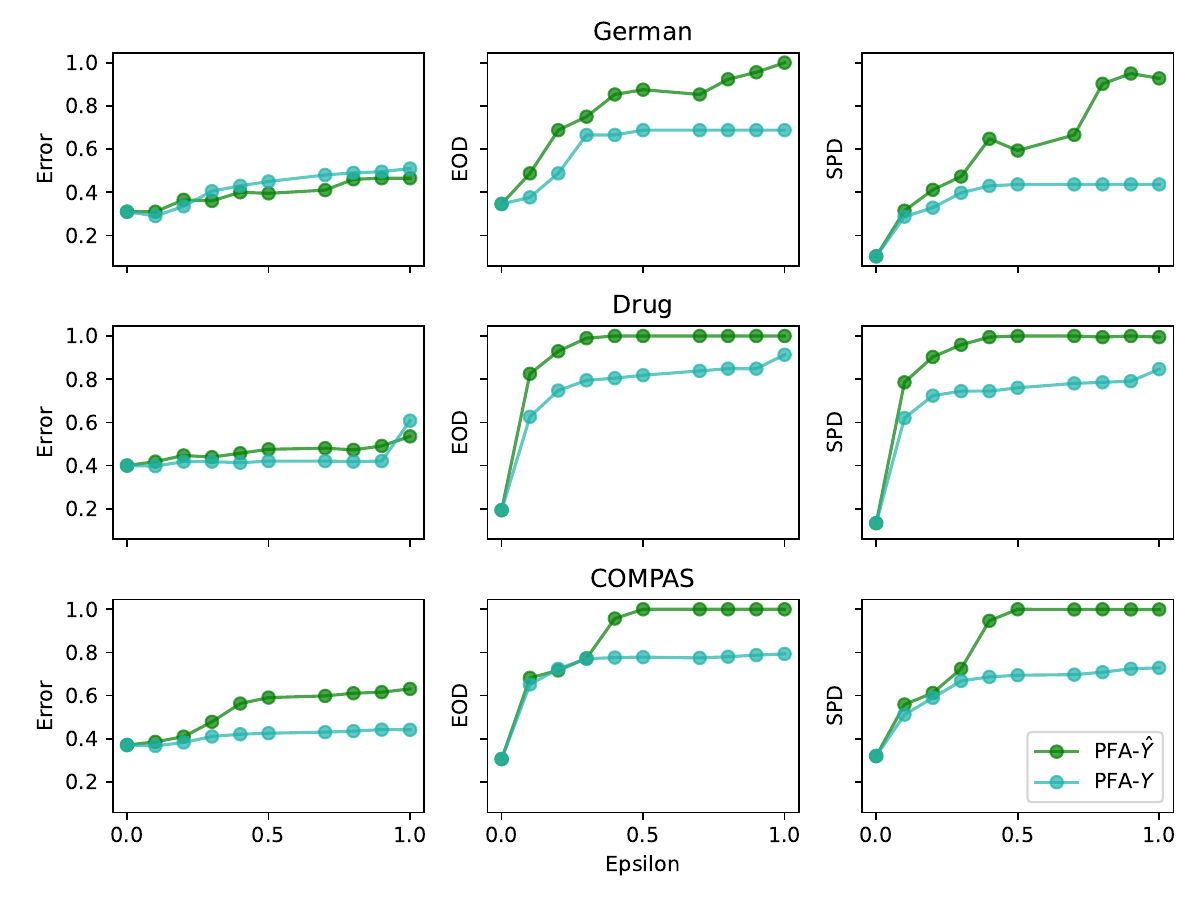}
        \caption{Logistic Regression} \label{fig:lrabl}
    \end{subfigure}
    
    \vfill
    
    \begin{subfigure}[b]{0.48\linewidth}
        \includegraphics[trim={0 0 0 0},clip,width=\linewidth]{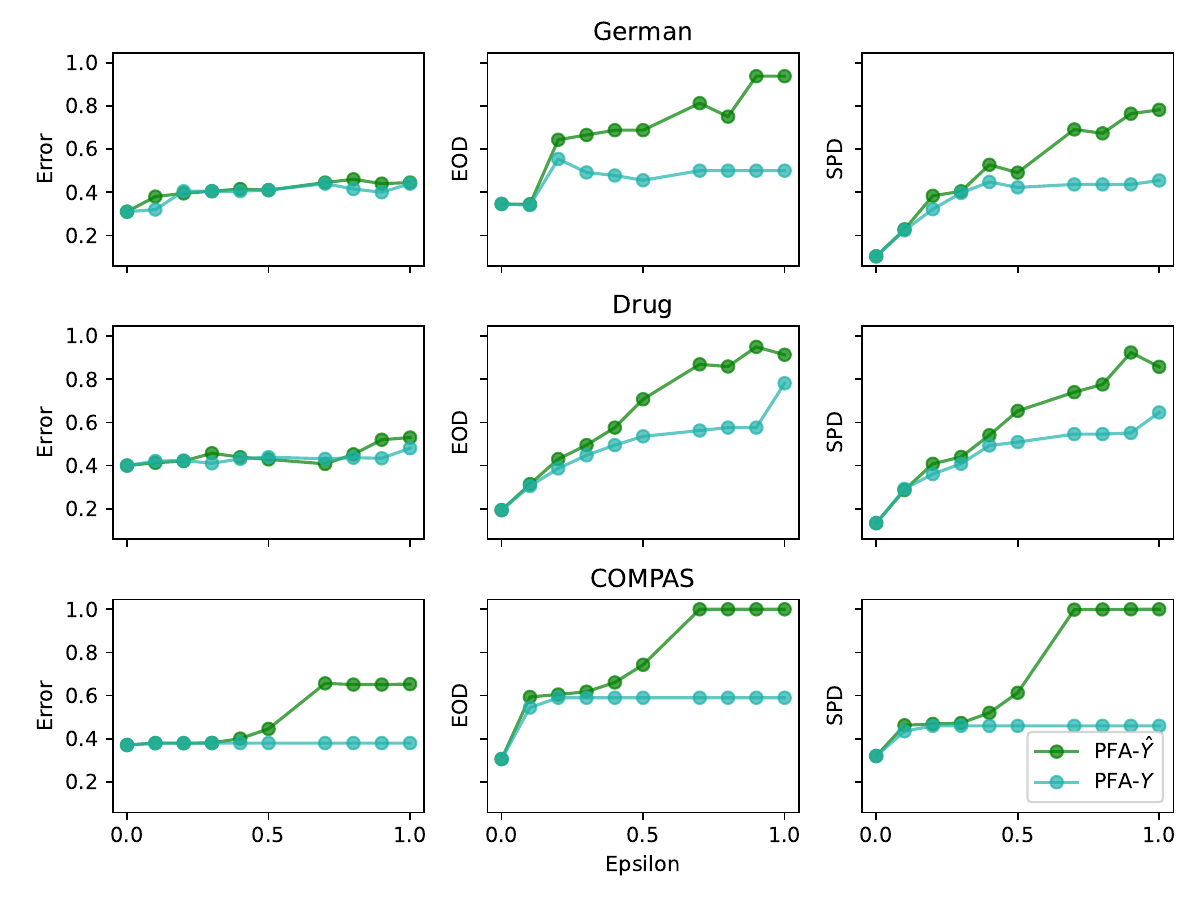}
        \caption{Decision Tree} \label{fig:dtabl}
    \end{subfigure}
    \hfill
    \begin{subfigure}[b]{0.48\linewidth}
        \includegraphics[trim={0 0 0 0},clip,width=\linewidth]{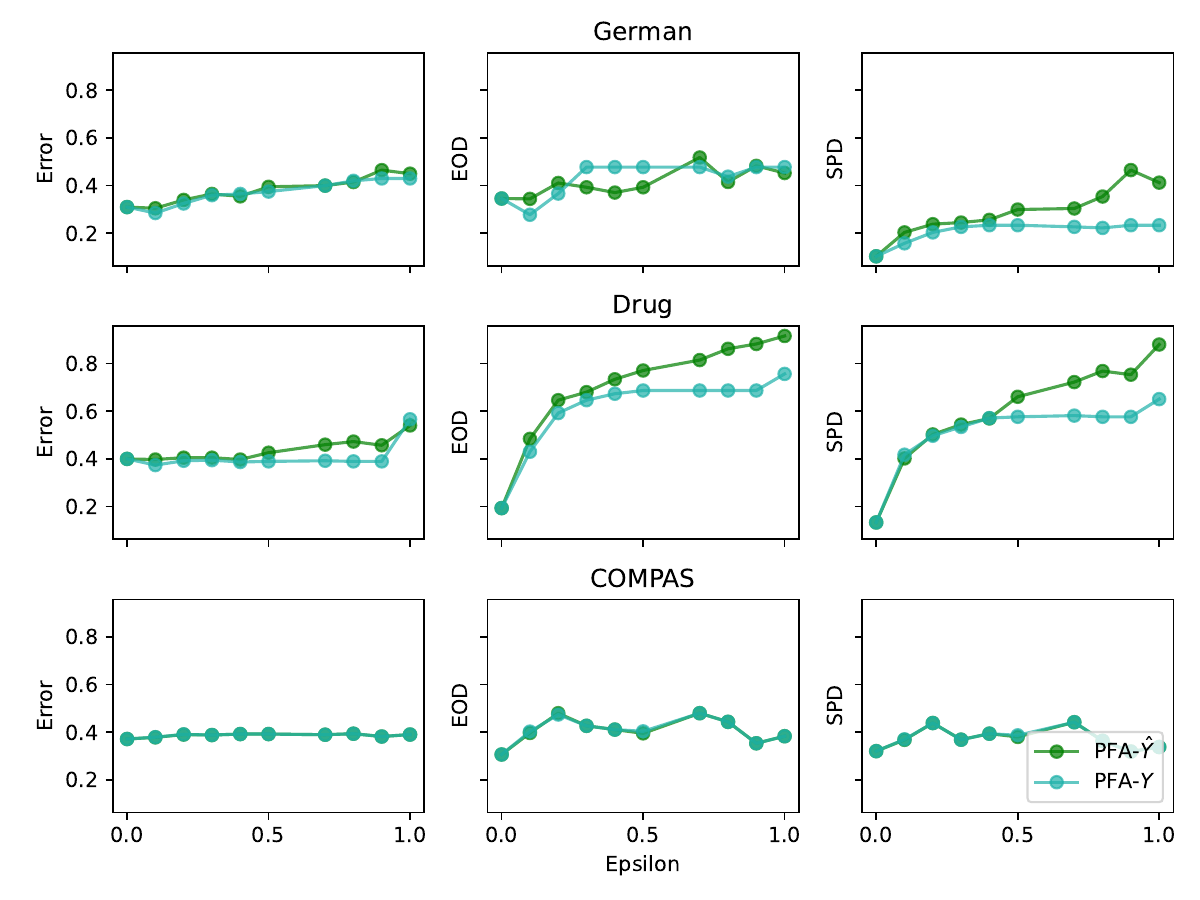}
        \caption{K-Nearest Neighbors} \label{fig:knnabl}
    \end{subfigure}
    \caption{Comparing our approach using $\hat{Y}$ versus $Y$ shows that the use of a surrogate model is essential for achieving a strong adversarial effect.} \label{fig:all_models_abl}
\end{figure*}
\raggedbottom
\begin{figure*}[t]
    \centering
    \begin{subfigure}[b]{0.48\linewidth}
        \includegraphics[trim={0 0 0 0},clip,width=\linewidth]{sections/figures/GNB_abl2.pdf}
        \caption{Gaussian Naive Bayes} \label{fig:gnbabl2}
    \end{subfigure}
    \hfill
    \begin{subfigure}[b]{0.48\linewidth}
        \includegraphics[trim={0 0 0 0},clip,width=\linewidth]{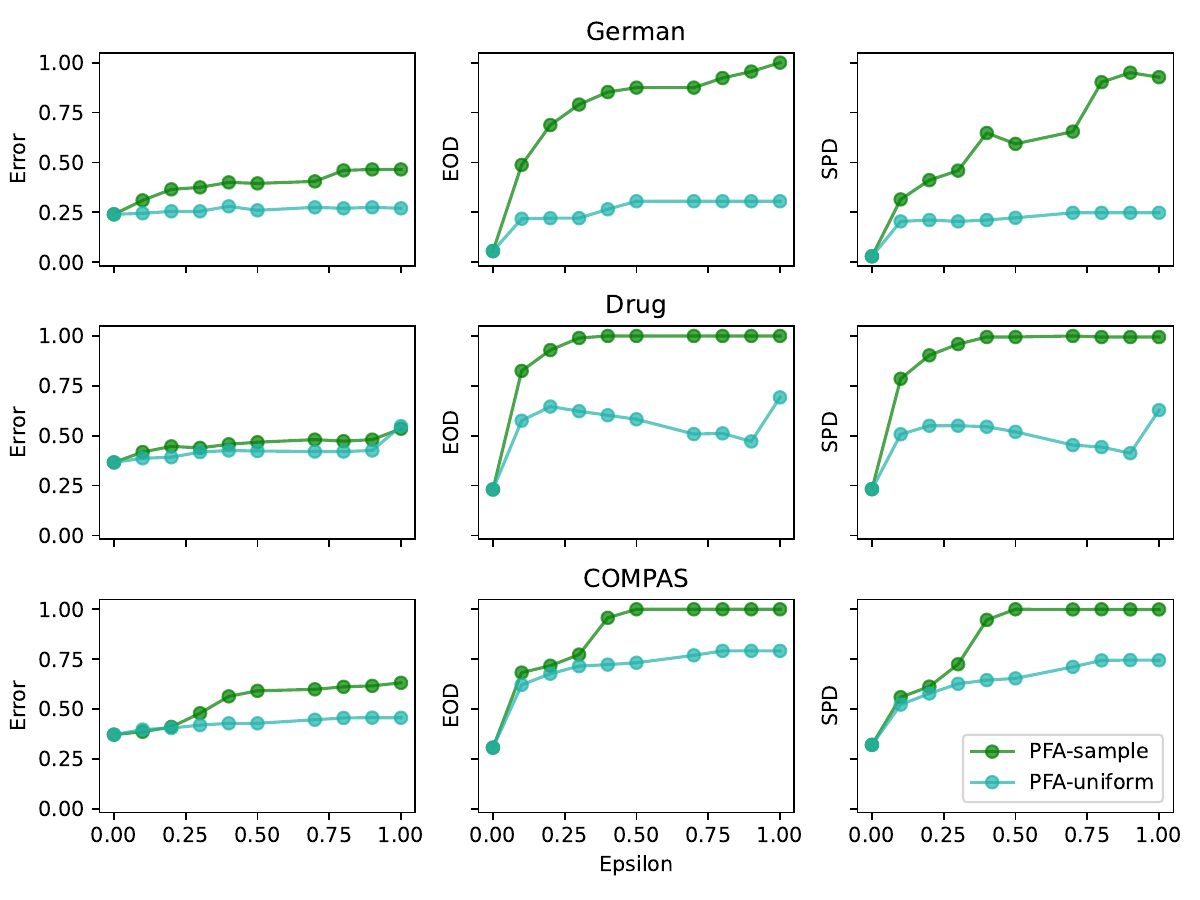}
        \caption{Logistic Regression} \label{fig:lrabl2}
    \end{subfigure}
    
    \vfill
    
    \begin{subfigure}[b]{0.48\linewidth}
        \includegraphics[trim={0 0 0 0},clip,width=\linewidth]{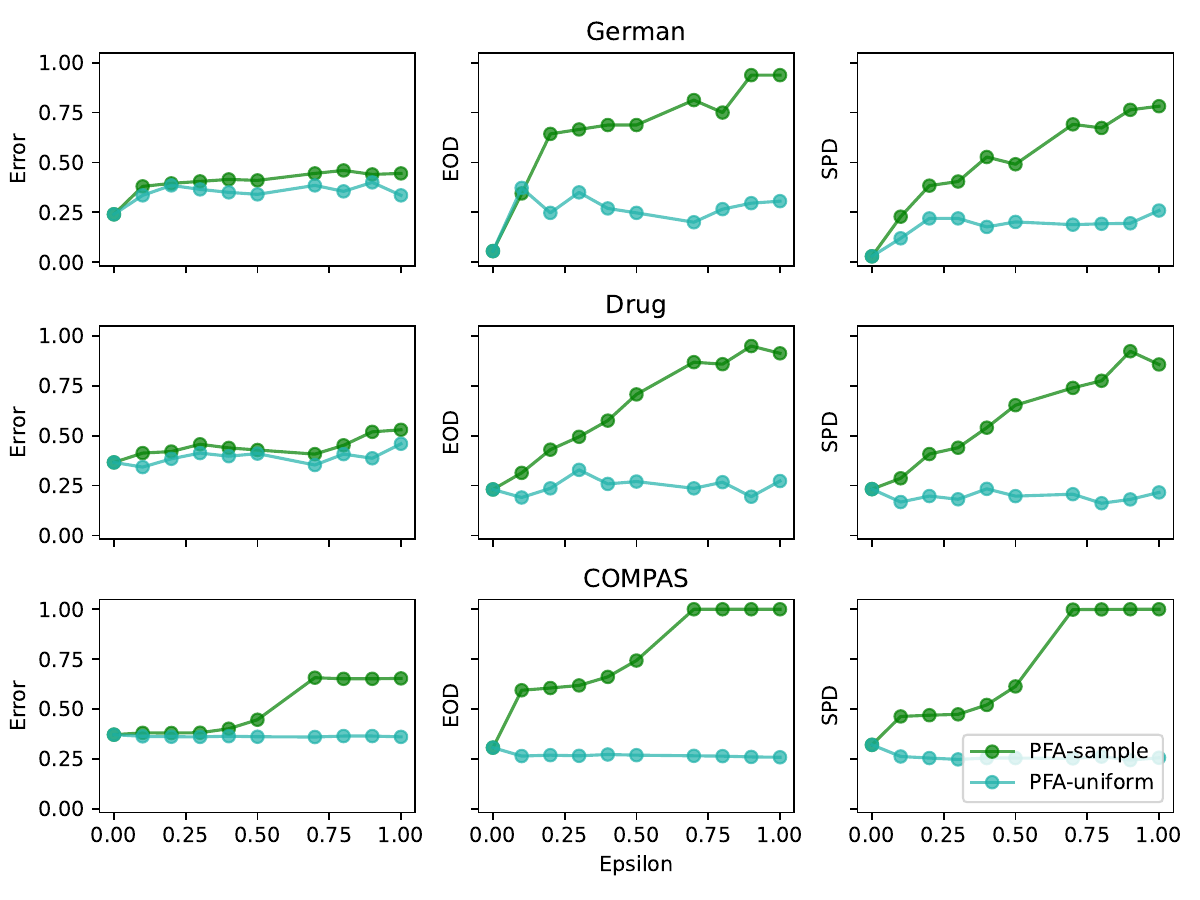}
        \caption{Decision Tree} \label{fig:dtabl2}
    \end{subfigure}
    \hfill
    \begin{subfigure}[b]{0.48\linewidth}
        \includegraphics[trim={0 0 0 0},clip,width=\linewidth]{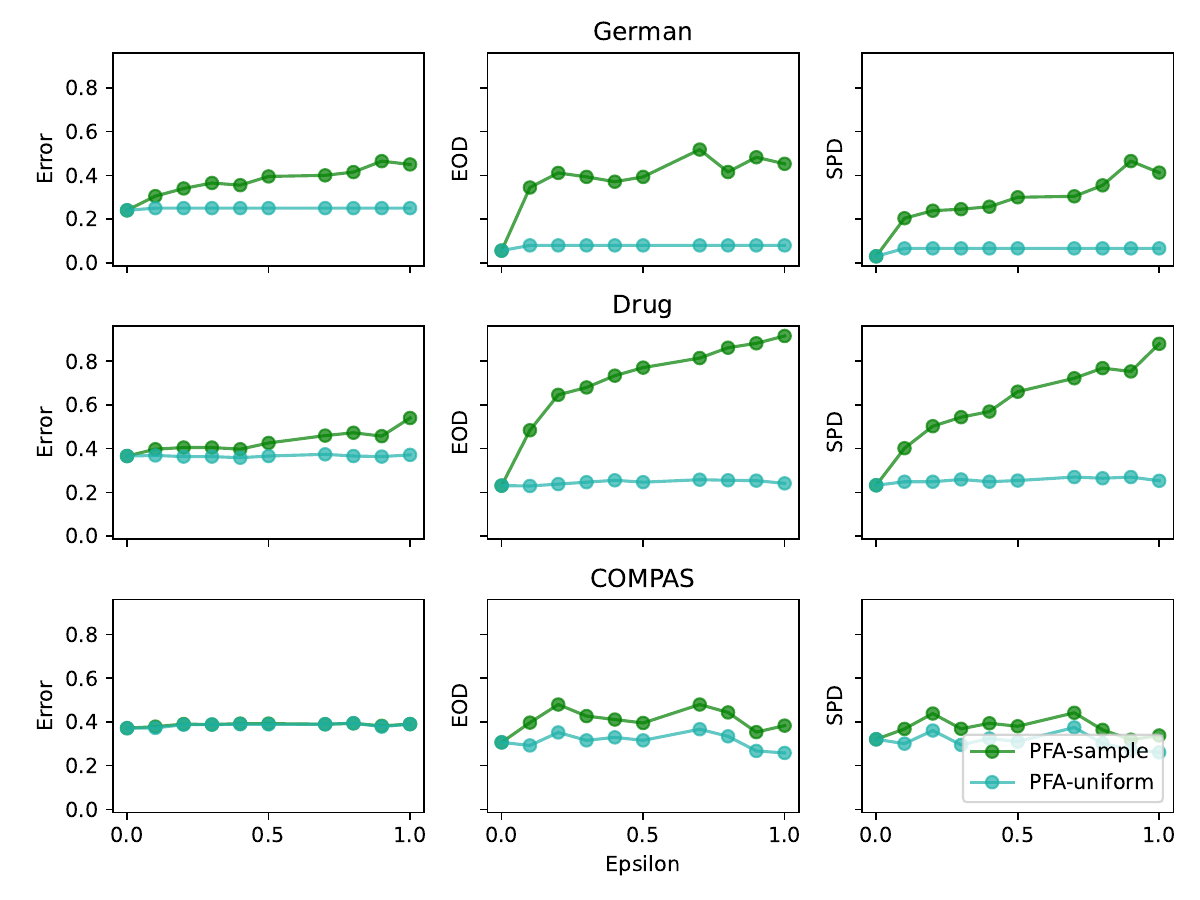}
        \caption{K-Nearest Neighbors} \label{fig:knnabl2}
    \end{subfigure}
    \caption{Comparing our approach using our heuristic sample-based poisoned sample generation method versus feasible poisoned sample generation method shows that the use of generating poisoned samples with sampling is much more efficient.} \label{fig:all_models_abl_2}
\end{figure*}
\raggedbottom
\end{appendices}

\end{document}